\documentclass[draftcls,noinfoline]{imsart} 
\RequirePackage[OT1]{fontenc}
\RequirePackage{amsthm,amsmath}
\RequirePackage[authoryear]{natbib}
\bibpunct{(}{)}{;}{a}{,}{,}

\usepackage[top=25mm,bottom=25mm,left=25mm,right=25mm]{geometry}

\RequirePackage[colorlinks,citecolor=blue,urlcolor=blue]{hyperref}

\RequirePackage{amssymb,bbm,enumerate,array,mathrsfs,amsfonts,float}
\RequirePackage{graphicx, graphics}
\RequirePackage{booktabs,blkarray}
\RequirePackage{color,xcolor}
\RequirePackage{picture}

\startlocaldefs
\numberwithin{equation}{section}
\theoremstyle{plain}
\newtheorem{theorem}{Theorem}[section]

\newtheorem{corollary}{Corollary}[section]
\newtheorem{lemma}{Lemma}[section]
\newtheorem{assumption}{Assumption}[section]
\newtheorem{remark}{Remark}[section]
\newtheorem{condition}{Condition}[section]
\endlocaldefs
\def\prob {{\rm P}}
\def\E{\mbox{\rm E}}

\def\O{\mbox{\rm O}}
\def\o{\mbox{\rm o}}
\renewcommand{\vec}[1]{\mbox{\boldmath ${#1}$}}
\def\vmu{\vec{\mu}}
\def\valpha{\vec{\alpha}}
\DeclareMathOperator*{\argmin}{arg\,min}

\begin{document}

\begin{frontmatter}
\title{Kernel Machines With Missing Responses}
\runtitle{Kernel Machines With Missing Responses}

\begin{aug}
\author{\ead[label=e1]{liutiantian99@yeah.net}}
\author{\ead[label=e2]{ygoldberg@stat.haifa.ac.il}}

\runauthor{Tiantian Liu And Yair Goldberg}

\affiliation{East China Normal University\thanksmark{m1} and The University of Haifa\thanksmark{m2}}

\address{Tiantian Liu\\School of Statistics\\
	 East China Normal University\\
	\printead{e1}}
	
\address{Yair Goldberg\\Department of Statistics
	\\The University of Haifa\\ \printead{e2}}

\end{aug}

\begin{abstract}
Missing responses is a missing data format in which outcomes are not always observed.
In this work we develop kernel machines that can handle missing responses. First, we propose a kernel machine family that uses mainly the complete cases. For the quadratic loss, we then propose a family of doubly-robust kernel machines. The proposed kernel-machine estimators
can be applied to both regression and classification problems. We prove oracle inequalities
for the finite-sample differences between the kernel machine risk and Bayes risk.
We use these oracle inequalities to prove consistency and to calculate convergence rates.
We demonstrate the performance of the two proposed kernel-machine families using both
a simulation study and a real-world data analysis.
\end{abstract}


\end{frontmatter}
\section{Introduction}\label{sec:intr}

We consider the problem of learning in the presence of missing responses.
Missing response is a type of missing data in which the response variable
cannot always be observed.
Missing responses are common in market research surveys,
medical research, and opinion polls.
Our motivating example is the Los Angeles
County homeless survey directed by the
Los Angeles Homeless Services Authority (LAHSA).
In the Los Angeles County there are 2054 tracts.
The LAHSA was interested in surveying the number of
homeless counts in the different tracts. For each tract,
information on the median household income,
the percentage of unoccupied housing units, etc.,
were collected. Due to budget constraints, LAHSA used stratified
spatial sampling of tracts to conduct the survey. There were 244
tracts which are known to have a large homeless population.
All of these tracts were included in the survey.
Out of the 1810 other tracts, 265 were randomly included in the survey,
leaving 1545 that where not included. The probability of tract inclusion in
the survey was dependent on the Service Provision Area (SPA).
Different areas have different probability of being visited.
Thus, this is a problem of missing responses, as covariates
were collected for  all tracts but responses were collected only for
tracts that were included in the survey. More details can be found
in \citet{Kriegler:Berk:10}.

Other examples of missing responses include the following.
Consider a clinical study in which genetic information is
collected on all participants but the level of a specific
biomarker is collected only on a subsample based on their
genetic information. In this example, the genetic information
are the covariates which are collected for everyone but the
biomarker which is the response missing for some of the participants.

Inference for missing data is challenging. There are three
main mechanisms leading to missing data, missing completely at random
(MCAR), missing at random (MAR), and not missing at random (NMAR),
see \citet[Chapter 1]{Little:Rubin:02}.
The two examples of missing responses discussed above can be cataloged as MAR.
In the homeless count example,
tracts are visited dependent on the areas and
independent of the actual counts in these tracts.
Biomarker levels are collected based
on the genetic profile and not the biomarker level.

Three main approaches are usually used for handling
missing values in statistical analysis.
The complete case analysis uses mainly the observations that have
no missing data. These observations are referred to as the
complete cases. The second approach is imputation
where a value or a set of values
are assigned to each missing value.
The third approach is a maximum likelihood approach
which first poses some models for response given the covariates.
Under MAR assumption, the likelihood function can be
written as the likelihood function of response-given covariates multiplied by
the likelihood function of covariates and missing mechanism, thus enabling maximizing
separately the two parts of the likelihood. \citet{Pelckmans05} summarize the three approaches;
see also \citet{Little:Rubin:02}, and \citet[Chapter 6]{Tsiatis2006Semiparametric}.

We develop a kernel-machine approach for missing responses.
Kernel methods, which include SVMs as a special case,
are easy-to-compute techniques that
enable estimation under weak or no assumptions
on the distribution \citep{Steinwart:Christmann:08,Hofmann:08}.
Kernel machines minimize a regularized version of empirical risk
where the empirical risk is the average
of a loss function on the observed sample.
In recent years, kernel methods have been developed for
many types of data including some missing data settings.
However, so far, no work has been done
on missing responses in the context of kernel machines.

We first propose a family of kernel machines that can be considered as inverse-weighted-probability complete-case estimators \citep{Robins:Rotnitzky:Zhao:94,Tsiatis2006Semiparametric}.
More specifically, we first model the missing mechanism
if it is unknown; in some settings it is known by design
and then we do not need to model it. Then,
we use the estimated inverse probabilities of the
observed cases to weight the loss function of the complete cases.
We show that if the missing mechanism is specified correctly,
the empirical risk, which is sum of the weighted loss function,
is an unbiased estimator for the risk. We then prove oracle inequalities,
consistency results, and calculate convergence rates
for this type of kernel machine. The main drawback of this approach is that the missing
mechanism is estimated using a model and when this model is misspecified,
the estimator could be biased.

We propose a doubly-robust kernel-machine estimator
in order to overcome the potential bias in missing mechanism
misspecification. Doubly-robust estimators are augmented
inverse-probability-weighted-complete-case estimators.
\citet{Scharfstein:Rotnitzky:99} first introduced
the notation of doubly-robust estimators. \citet{Bang:Robins:05} give an overview of
the development of doubly-robust estimators. \citet{Zhao:Zeng:15}
give a new application in the setting of individualized treatment regimes.
We face two main challenges when
	constructing a doubly-robust kernel-machine estimator. The first is that the loss function obtained by adding an augmentation term
	needs to ensure both doubly-robustness and convexity. The second is that the augmentation-term estimation needs to converge uniformly over a set of functions that grows with the sample size. We are not aware of any doubly-robust estimators
in the context of kernel machines.
To compute the proposed doubly-robust
kernel-machine estimator, we first estimate the missing mechanism
and also the conditional distribution response given covariates.
The latter is used to calculate the conditional risk.
Then, based on the previous weighted loss function,
we augment a weighted conditional risk.
Empirical risk based on this loss function has
the doubly robust property in the sense that if
either the missing mechanism or the conditional
distribution is correctly specified, not necessarily
both, the empirical risk is unbiased. We prove oracle
inequalities and consistency results, and calculate convergence rates for
quadratic-loss doubly-robust kernel machines.

To illustrate the proposed kernel machine methods, we
apply them to simulated data.
In the simulation study, we use the proposed kernel machine methods
to analyze regression and classification problems. We then analyze
the Los Angeles homeless data, comparing the proposed kernel
machines to other existing methods.

Approaches for missing responses include the work of
\citet{Wang:Rao:02}.
Under the missing at random assumption,
they first imputed the missing response values by the kernel
regression imputation and then constructed a complete data empirical
likelihood to obtain the mean of the response variable from the imputed data set.
\citet{Wang:H:04} extended a semiparametric regression
analysis method to include missing responses.
Their interest was to
estimate the mean of the response. First they used a partially linear
semiparametric regression model to estimate the conditional mean of
response-given covariates; only completed cases are included
in this step. Then they used weighted observed responses and weighted
conditional mean of response to estimate the mean.
\citet{Smola:Vishwanathan:Hofmann:05} developed a framework
in which kernel methods can
be written as estimators in an exponential family,
which can handle both missing covariates and missing responses.
They extended the concave convex procedure \citep{Yuille:Rangarajan:03}
to find a local optimum. However, there is no guarantee for convergence and
the computations can be demanding.
\citet{Liang:Wang:07} proposed a partially linear model
for missing responses with measurement errors on the covariates.
\citet{Azriel:Brown:16} studied a regression problem with missing
responses. They showed that when the conditional expectation is
not linear in the predictors, the additional observations provide more information.
In their work, they constructed the best linear predictor which depends
also on the incomplete data.

In the learning literature, semi-supervised learning is halfway between
supervised and unsupervised learning and it can be used to handle missing data.
In this learning scenario, a dataset has two components: labeled and unlabeled.
Semi-supervised learning wishes to have a more accurate prediction
by taking into account also the unlabeled data. Semi-supervised learning
method uses, for example, distance measures to create clustering and
neighbouring graphs.
Then, the semi-supervised method uses the obtained structure to get a better
understanding of the labeled data. However, the semi-supervised approach is
different from our proposed kernel machine method,
because the semi-supervised learning methods do not
consider the missing mechanism and do not try to account
for the bias of the complete observations.
Details about semi-supervised learning
is given in \citet{Chapelle:Scholkopf:Zien:06}.

The paper is organized as follows. Background and notation are given in Section~$\ref{sec:preli}$. In Section~$\ref{sec:svm}$
we present the proposed kernel machines. Section~$\ref{sec:res}$ presents
the main theoretical results, including the oracle inequalities, consistency results,
and convergence rate calculations.
Simulation results are shown
in Section~$\ref{sec:sim}$. The Los Angeles homeless data is analyzed
in Section~$\ref{sec:ReaDat}$. In Section~$\ref{sec:conDis}$ we
discuss potential future directions. Technical proofs appear in the Supplementary Material.

\section{Preliminaries}\label{sec:preli}
Assume that $n$ independent
and identically distributed observations
$D=\{(M_1,X_1,Y_1),\ldots,\allowbreak(M_n,X_n,Y_n)\}$ are collected.
Here, $M$ is a missingness indicator such that $M=1$ if $Y$ is observed,
and $M=0$ otherwise. The random vector $X$ is a covariate vector that takes its values in a compact set $\mathcal{X}\subset\mathbb{R}^{d}$. The random variable $Y$ is the response that takes its values in the set $\mathcal{Y}\subset\mathbb{R}$ where $\mathcal{Y}$
can be, for example $\{-1,1\}$ for classification problems, and some compact segment of $\mathbb{R}$ for regression problems. Define $\pi(X)=\prob(M=1\mid X)$ as the propensity score.

We need the following assumption which is discussed in \citet{Tsiatis2006Semiparametric}.

\begin{assumption}\label{MAR}
The missing mechanism is MAR and there is a positive constant $0<c<\frac{1}{2}$  such that $\inf_{x\in\mathcal{X}} \pi(x)\geq 2c>0$.
\end{assumption}

Let $\mathcal{P}$ be the set of all probability measures that follow Assumption~\ref{MAR}.
In the following, we will focus our analysis on probability measures in $\mathcal{P}$.

We now move to discuss kernel machine learning methods. Let $L:\mathcal{Y}\times\mathbb{R}\mapsto[0,\infty)$ be a loss function where $L(Y,f(X))$ can be interpreted as the cost of predicting $Y$ by $f(X)$.
We assume that $L$ is convex and a locally Lipschitz continuous loss function such that for all $a>0$ there exists a constant $C_{L}(a)\geq0$ for which
\[\sup_{ y\in\mathcal{Y}}|L(y,t)-L(y,t')|\leq C_{L}(a)|t-t'|,\qquad\qquad t,t'\in[-a,a].\]
We also assume that $L(y,0)$ is bounded and without loss of generality we assume that $L(y,0)\leq 1$. Define the $L$-risk $R_{L,\text{\prob}}(f)\equiv\E[L(Y,f(X))]$ to be the expected loss when using the function $f(X)$ as a predictor of $Y$.
Define the the Bayes risk as $R_{L,\text{\prob}}^{\ast}\equiv\inf_{f \text{ is measurable}}R_{L,\text{\prob}}(f)$, where the Bayes risk is the smallest possible risk.  The empirical risk is defined by
\[R_{L,D}(f)\equiv\frac{1}{n}\sum_{i=1}^{n}L(Y_i,f(X_i)).\]

Let $\mathcal{H}$ be a separable reproducing kernel Hilbert space (RKHS) of a bounded measurable kernel on
$\mathcal{X}$ and denote its norm by $\|\cdot\|_{\mathcal{H}}$.
Let $k:\mathcal{X} \times \mathcal{X}\mapsto\mathbb{R}$ be its reproducing kernel. We assume that
$k$ is a universal kernel, which means that $\mathcal{H}$ is dense in the space of bounded continuous functions with respect to the supremum norm and that $\|k\|_{\infty}\leq 1$ \citep[see Chapter 4 of][for details]{Steinwart:Christmann:08,Hofmann:08}.
A kernel machine $f_{D,\lambda}$ is the minimizer of the regularized empirical risk,
\begin{equation}\label{DecFun}
f_{D,\lambda}\equiv\min _{f\in \mathcal{H}}\lambda\|f\|_{\mathcal{H}}^{2}+R_{L,D}(f),
\end{equation}
where the regularization term $\lambda\|f\|_{\mathcal{H}}^{2}$ penalizes the RKHS norm of $f$.

Since $L$ is a convex loss function, it can be shown \citep[Theorem 5.5 (Representer Theorem)]{Steinwart:Christmann:08} that
there is unique minimizer to the minimization problem~\eqref{DecFun}. Moreover,
this minimizer is of the form
\begin{equation}\label{Mini}
f_{D,\lambda}(x)=\sum_{i=1}^{n} \alpha_i k(x_i,x),
\end{equation}
where $\valpha=(\alpha_1,\cdots,\alpha_n)^{\intercal}\in \mathbb{R}^{n}$ is a vector of coefficients.

\section{Kernel machines with missing responses}\label{sec:svm}
In this section, we derive two types of kernel machines. The first type uses
weighted-complete-cases while the second type has the doubly-robust property.

\subsection{Weighted-complete-case kernel machines}

Let $\widehat{\pi}(X)>0$ be an estimator of
the propensity score.
Note that a naive complete case estimator for $R_{L,P}(f)$, denoted by $R_{L,D}(f)$ is
\[\frac{\sum_{i=1}^{n}M_iL(Y_i,f(X_i))}{\sum_{i=1}^{n}M_i}=
\frac{n^{-1}\sum_{i=1}^{n}M_iL(Y_i,f(X_i))}{n^{-1}\sum_{i=1}^{n}M_i}
\stackrel{\text{P}}\longrightarrow \frac{E[ML(Y,f(X))]}{\E[M]},\]
where
\[\frac{\E[ML(Y,f(X))]}{\E[M]}=\E[L(Y,f(X))],\]
if and only if $\E[ML(Y,f(X))]= \E[M]\E[L(Y,f(X))]$.
This is a restrictive condition which typically requires $M$ to be independent of the pair $(X,Y)$.
Therefore, consistency of $R_{L,D}(f)$ to $R_{L,P}(f)$ cannot be
guaranteed when using a naive complete case estimator.

Using Assumption~\ref{MAR},
\begin{align}\label{eq:IPW_EQ}
  \E\left\{\frac{ML(Y,f(X))}{\pi(X)}\right\}
  & =\E\left[\E\left\{\frac{ML(Y,f(X))}
  {\pi(X)}\mid X\right\}\right] \\ \notag
  &=\E\left[\E\{L(Y,f(X))\mid X\}\E\left\{\frac{M}{\pi(X)}\mid X\right\}\right]\\ \notag
  &=\E\left\{\E(L(Y,f(X))\mid X)\right\}\\ \notag
  &=\E(L(Y,f(X))).
\end{align}
Thus, in order to avoid this bias, we propose to weight the complete cases appropriately. Let $\Pi$ be the set of conditional distribution of $M$ given $X$.
Define the weighted loss function for missing
response data  $L_{W}:\Pi\times\{0,1\}\times\mathcal{X}\times\mathcal{Y}\times\mathbb{R}\mapsto[0,\infty)$  as
\[L_{W}\left(\pi^{\ast},M,X,Y,f(X)\right)\equiv\frac{ML(Y,f(X))}{\pi^{\ast}(X)}
=\begin{cases}
  \frac{L(Y,f(X))}{\pi^{\ast}(X)} & M=1,\\
   0    & M=0.
\end{cases}\]
Define the weighted empirical risk as
\[R_{L_{\widehat{W}},D}(f)\equiv\frac{1}{n}\sum_{i=1}^{n} L_{W}\left(\widehat{\pi},M_i,X_i,Y_i,f(X_i)\right)=\frac{1}{n}\sum_{i=1}^{n}\frac{M_i L(Y_i,f(X_i))}{\widehat{\pi}(X_i)}.\]

Since $L(Y, f(X))$ is a convex function, both $L_{W}\left(\pi^{\ast},M,X,Y,f(X)\right)$
and $L_{W}(\widehat{\pi},\allowbreak M, X,Y,f(X))$ are convex functions.
The missing-response kernel machine is defined as
\begin{equation}\label{eq:cons_decfun}
  f_{D,\lambda}^{W}\equiv\argmin_{f\in \mathcal{H}}\lambda\|f\|^{2}_{\mathcal{H}}+R_{L_{\widehat{W}},D}(f)=\argmin_{f\in \mathcal{H}}\lambda\|f\|^{2}_{\mathcal{H}}+\frac{1}{n}\sum_{i=1}^{n}\frac{M_i L(Y_i,f(X_i))}{\widehat{\pi}(X_i)}.
\end{equation}

\begin{lemma}\label{lemma1}
Assume that the conditional probability estimator $\widehat{\pi}(x)$ converges to $\pi(x)$
in probability and that Assumption~\ref{MAR} holds. Then,
for any given $f\in\mathcal{H}$, the weighted empirical risk $R_{L_{\widehat{W}},D}(f)$ is
a consistent estimator for the risk $R_{L,\text{\prob}}(f)$.
\end{lemma}

\begin{proof}
Since $\widehat{\pi}(X)$ is a consistent estimator for $\pi(X)$,
\begin{align*}
  R_{L_{\widehat{W}},D}(f)
  &= \frac{1}{n}\sum_{i=1}^n\left[\frac{M_iL(Y_i,f(X_i))}{\widehat{\pi}(X_i)}\right]
  =  \frac{1}{n}\sum_{i=1}^n\left[\frac{M_iL(Y_i,f(X_i))}{\pi(X_i)}\right]+\o_p(1)\\
  &\stackrel{\text{P}}
  \longrightarrow \E\left\{\frac{ML(Y,f(X))}{\pi(X)}\right\}.
\end{align*}
The result then follows since $\E[ML(Y,f(X))/\pi(X)]=R_{L,P}(f)$ by~\eqref{eq:IPW_EQ}.
\end{proof}

Note that $R_{L_{\widehat W},D}(f)$ is a consistent
estimator of $R_{L,\text{\prob}}(f)$ only under the
assumption that ${\widehat{\pi}(X)}$ is consistent for $\pi(X)$.
Since this assumption cannot be verified, we also develop a family of
doubly-robust kernel machines.

\subsection{Doubly-robust kernel machines}

Suppose the conditional distribution of $Y$ given $X$ is $F_{Y|X}\left(y\mid X,\beta_{0}\right)$,
where $\beta_{0}$ $\in$ $\mathbb{B}$ is an unknown parameter and $\mathbb{B}$ is a parameter space. Assume that $F_{Y|X}(y\mid x,\beta)$ is continuously differentiable with respect to $\beta$ for
$x\in\mathcal{X}$, $y\in\mathcal{Y}$. Let
\begin{equation}\label{eq:CondiExpct}
H(x,\beta_0,f(x))=\E\left\{L(Y,f(X))\mid X=x\right\}
=\int_{y\in\mathcal{Y}} L(y,f(x))dF_{Y|X}\left(y\mid x,\beta_{0}\right).
\end{equation}
Using conditional expectations $\E\left\{H(X,\beta_0,f(X))\right\}=\E\left\{L(Y,f(X))\right\}=R_{L,\prob}(f)$.

Let $\widehat\beta$ be an estimator of $\beta_0$, and define $\widehat{H}(X,f(X))=H(X,\widehat\beta,f(X))$. Assume that $\widehat\beta\stackrel{\text{\prob}}\longrightarrow \beta^{\ast}$, where $\beta^{\ast}$ $\in$ $\mathbb{B}$.
Here, $\beta^{\ast}$ does not necessarily equal $\beta_0$. By the continuous mapping theorem,
\[F_{Y|X}\left(y\mid x,\widehat\beta\right)\stackrel{\text{P}}\longrightarrow
 F_{Y|X}\left(y\mid x,\beta^{\ast}\right),~\text{for all $x\in\mathcal{X}$ and $y\in\mathcal{Y}$.}\]
Furthermore, we assume that $H(x,\beta,f(x))$ is a continuous function of $\beta$ for
every fixed $x\in\mathcal{X}$.
Define the following augmented loss function
\[L_{W,H}\left(\pi^{\ast},H^{\ast},M,X,Y,f(X)\right)\equiv\frac{ML(Y,f(X))}{\pi^{\ast}(X)}-
\frac{M-\pi^{\ast}(X)}{\pi^{\ast}(X)}H(X,\beta^{\ast},f(X)).\]
This function doesn't need to be nonnegative as opposed to $L$ and $L_W$. The corresponding empirical risk is
\begin{align*}
R_{L_{\widehat W,\widehat{H}},D}(f)
&\equiv\frac{1}{n}\sum_{i=1}^nL_{W,H}\left(\widehat{\pi},\widehat{H},M,X,Y,f(X)\right)\\
&=\frac{1}{n}\sum_{i=1}^n
\left\{\frac{M_iL(Y_i,f(X_i))}{\widehat{\pi}(X_i)}-
\frac{M_i-\widehat{\pi}(X_i)}{\widehat{\pi}(X_i)}\widehat{H}(X_i,f(X_i))\right\}.
\end{align*}

In order to define the doubly-robust estimator,
we need both $L_{W,H}$ and $L_{\widehat{W},\widehat{H}}$  to be convex functions.
\begin{lemma}\label{lemma3}
Let $L(Y,f(X))$ be the  quadratic loss, that is $L(Y,f(X))=(Y-f(X))^{2}$.
Then, $L_{W,H}$ and $L_{\widehat{W},\widehat{H}}$ are both convex functions.
\end{lemma}

The doubly-robust kernel machine is defined as
\begin{equation}\label{eq:DRdecfun}
  f_{D,\lambda}^{DR}\equiv\argmin_{f\in \mathcal{H}}\lambda\|f\|^{2}_{\mathcal{H}}+R_{L_{\widehat W,\widehat H},D}(f).
\end{equation}
We assume that
$\widehat{\pi}(X)$ converges in probability to some conditional
probability function $\pi^{\ast}(X)$, not necessarily the true $\pi_0(X)$. We also assume that  $\widehat\beta\stackrel{\text{\prob}}\longrightarrow \beta^{\ast}$ which does not necessarily equal $\beta_{0}$. It follows that
\[\widehat{H}(x,f(x))\stackrel{\text{P}}\longrightarrow  H(x,\beta^{\ast},f(x)),\]
where $H(x,\beta^{\ast},f(x))=\int_{y\in\mathcal{Y}} L(y,f(x))dF_{Y|X}(y\mid x,\beta^{\ast})$. The following lemma states that if either the estimators of $\pi(X)$ or the estimator of $\beta_0$ is consistent, the doubly-robust empirical risk $R_{L_{\widehat W,\widehat H},D}(f)$ is consistent for $R_{L,\text{P}}(f)$.
\begin{lemma}\label{lemma2}
For any given $f\in\mathcal{H}$, if either $\pi^{\ast}(X)=\pi(X)$ or $\beta^{\ast}=\beta_0$, then
$R_{L_{\widehat W,\widehat H},D}(f)\stackrel{\text{P}}\longrightarrow R_{L,\prob}(f)$.
\end{lemma}

\subsection{Estimation of the augmentation term}

We present two explicit examples of estimation of the augmentation term which is needed for the doubly-robust estimation. We limit the discussion to the quadratic loss function.

\subsubsection{Regression}

Consider the following location-shift regression model \citep[Chapter~5]{Tsiatis2006Semiparametric}
\[Y=\mu(X,\beta_0)+\varepsilon,\]
where $\varepsilon$ is the error with mean zero independent of $X$. For example, when the model is
a linear regression, $\mu(X,\beta_0)=X^{T}\beta_0$;
 when the model is a log single index model, $\mu(X,\beta_0)=\log\left(X^{T}\beta_0\right)$. Write
We present two explicit examples of estimation of the augmentation term which is needed for the doubly-robust estimation. We limit the discussion to the quadratic loss function.
\begin{eqnarray*}
 H(X,\beta_0,f(X))
 &=&\E(L(Y,f(X))\mid X)\\
 &=&\E\left\{\left(\mu(X,\beta_0)+\varepsilon-f(X)\right)^{2}\mid X\right\}\\
 &=&\E\left\{\mu(X,\beta_0)-f(X))^{2}\mid X\right\}+\E\left(\varepsilon^{2}\mid X\right)\\
 &=&\left(\mu(X,\beta_0)-f(X)\right)^{2}+\E\left(\varepsilon^{2}\right),
\end{eqnarray*}
where the third equality holds because the error $\varepsilon$ has a mean
zero and is independent of $X$.

Two terms need to be estimated, namely,
$\beta_0$ and $\E\left(\varepsilon^{2}\right)$.
We first estimate $\beta_0$ by maximizing the likelihood function.

Let $F_{M,X}(m,x)$ and $F(m,x,y,\beta_0)$ denote the joint distribution of $(M,X)$ and $(M,X,Y)$ respectively. Then
\begin{align*}
 F(m,x,y,\beta_0)
=& F_{Y|M,X}(y\mid m,x,\beta_0)F_{M,X}(m,x)\\
=& F_{Y|X}(y\mid m,\beta_0)F_{M,X}(m,x),
\end{align*}
where the second equation follows from Assumption~\ref{MAR}.

Without loss of generality, suppose that the first $n_1$ triples $(M_i,X_i,Y_i)$
are the complete cases, and for the last $n-n_1$ observations,
only the covariates $X_i$ are observed.

The likelihood function can be written as
\begin{align}\label{Lihood}\notag
  \text{Likelihood}(D,\beta_0) &= \prod_{i=1}^{n_{1}}F(m_i,x_i,y_i,\beta_0)\prod_{i=n_{1}+1}^{n}F_{M,X}(m_i,x_i) \\
    &= \prod_{i=1}^{n_{1}}F_{Y|X}(y_i\mid x_i,\beta_0)F_{M,X}(m_i,x_i)\prod_{i=n_{1}+1}^{n}F_{M,X}(m_i,x_i) \\\notag
    &= \prod_{i=1}^{n_{1}}F_{Y|X}(y_i\mid x_i,\beta_0)\prod_{i=1}^{n}F_{M,X}(m_i,x_i).
\end{align}
Note that only the first term involves $\beta_0$ and hence it is enough to maximize
$\prod_{i=1}^{n_{1}}F_{Y|X}(y_i\mid \beta_0, x_i)$.
Given an estimate for $\beta_0$, we use\\
$\frac{1}{\#\{M_{i}=1\}}\sum_{i=1}^{n}M_i\left(Y_{i}-\mu\left(X_i,\widehat{\beta}\right)\right)^{2}=
\frac{1}{\#\{M_{i}=1\}}\sum_{i=1}^{n}M_i\widehat{\varepsilon}_{i}^{2}$ to estimate
$\E\left(\varepsilon^{2}\right)$. Note that this is a consistent estimator of
$\E\left(\varepsilon^{2}\right)$ when $\mu(X,\widehat{\beta})$ is a consistent estimator of $\mu(X,\beta_0)$,
since $X$ is independent of $\varepsilon$. Thus, plugging in
$\frac{1}{\#\{M_{i}=1\}}\sum_{i=1}^{n}M_i\widehat{\varepsilon}_{i}^{2}$ into $H(X,\beta_0,f(X))$, we obtain
\[\widehat{H}(X,f(X))=\left(\mu\left(X,\widehat{\beta}\right)-f(X)\right)^{2}+
\frac{1}{\#\{M_{i}=1\}}\sum_{i=1}^{n}M_i\left(Y_{i}-\mu\left(X_i,\widehat{\beta}\right)\right)^{2},\]
which can be a substitute in
 $R_{L_{\widehat W,\widehat H},D}(f)$. Minimizing $\eqref{eq:DRdecfun}$ with respect
to this $\widehat H$ results in the doubly-robust kernel machines for regression problems.

\subsubsection{Classification}

For classification problems, where $Y \in \{-1,1\}$, assume that the probability of $Y$ given $X$ follows a logistic model. More specifically, assume
that $\prob(Y=1\mid X,\beta_0)=\frac{\exp(X^{\intercal}\beta_0)}{1+\exp(X^{\intercal}\beta_0)}=\text{logit}(X,\beta_0)$. Write
\begin{eqnarray*}
 H(X,\beta,f(X))&=&\E(L(Y,f(X))\mid X)\\
 &=&\E\left(Y^{2}+f(X)^{2}-2f(X)Y\mid X\right)\\
 &=&1+f(X)^{2}-2f(X)\E\left(Y\mid X\right)\\
 &=&1+f(X)^{2}-2f(X)\left\{2\prob(Y=1\mid X,\beta_0)-1\right\}.
\end{eqnarray*}
Only the term $\prob(Y=1\mid X)$ needs to be estimated and using the logistic model,
\[\widehat{H}(X,f(X))=1+f(X)^{2}+2f(X)
-4f(X)\text{logit}\left(X,\widehat{\beta}\right).\]
Using the same argument as previous, it is enough to maximize
\[\prod_{i=1}^{n_{1}}F_{Y|X}(y_i,\beta_0\mid x_i)=
\prod_{i=1}^{n_{1}}\prob\left(Y=1\mid x_i,\beta_0\right)^{\frac{y_i+1}{2}}
\left\{1-\prob(Y=1\mid x_i,\beta_0)\right\}^{\frac{1-y_i}{2}},\]
since $Y$ is a binary variable that gets values in $\{-1,1\}$.
This is a standard logistic regression and the estimator of $\beta_0 $ can be found using standard tools.
Substituting $\widehat{H}(X,f(X))$
into $R_{L_{\widehat W,\widehat H},D}(f)$ and  minimizing $\eqref{eq:DRdecfun}$ results in the doubly-robust kernel machines for classification problems.

\subsection{Least-squares kernel machines with missing responses}\label{sec:km}
For the quad-ratic loss function,
 the corresponding kernel machines can be calculated explicitly.
The detailed calculations can be found in the Appendix.
Let $\valpha=(\alpha_1,\cdots,\alpha_n)^{\intercal}$, ${\bf Y}=(Y_1,\cdots,Y_n)^{\intercal}$,
$W=\text{diag}\left(M_1/\widehat{\pi}(X_1),\cdots,\allowbreak M_n/\widehat{\pi}(X_n)\right)$,
 and $A=W^{1/2}$. Define the kernel matrix $K$ as
$K_{ij}=k(X_i,X_j)$, for $i,j=1,\ldots,n$. Let $I$ be the $n\times n$ identity matrix.

For the weighted-complete-case kernel machines, the coefficient vector is
\[\widehat{\valpha}=(\lambda I+WK)^{-1}W{\bf Y}.\]
For the doubly-robust estimator,
\[\widehat{\valpha}=(K+\lambda I)^{-1}\left(WY+(I-W)\vmu\left(X,\widehat{\beta}\right)\right),\]
where $\vmu\left(X,\widehat{\beta}\right)=
\left(\mu\left(X_1,\widehat{\beta}\right),\cdots,\mu\left(X_n,\widehat{\beta}\right)\right)^{\intercal}$ is the vector of means. For the classification problem, the function $\mu$ is defined as
\[\mu\left(X,\widehat{\beta}\right)=2\text{logit}\left(X,\widehat{\beta}\right)-1.\]
In both cases, the kernel machines are defined as
\begin{equation*}
\widehat{f}_{D,\lambda}(x)=\sum_{i=1}^{n}\widehat{\alpha}_ik(x,X_i).
\end{equation*}

\section{Theoretical results}\label{sec:res}

\subsection{Assumptions, conditions and errors}\label{subsec:ass_con_err}
In Section \ref{sec:svm} we proved that for any
given $f\in\mathcal{H}$, the empirical risk based on the two proposed kernel
machines are consistent estimators of the risk function $R_{L,\prob}(f)$.
In this section, we prove universal consistency and derive the learning
rates of the proposed kernel machines.
Here, universal consistency means when the training set is sufficiently large,
the learning methods produce
nearly optimal decision functions with high probability for all $P\in\mathcal{P}$.
Learning rates provide a framework that is more closely related to practical needs.
It answers how fast $R_{L,\prob}(f_{D,\lambda})$ converges to the
Bayes risk  $R_{L,\prob}^{\ast}$. The learning rate of learning method is
defined in \citet[Lemma 6.5]{Steinwart:Christmann:08}.

In order to prove the universal consistency, we will
prove oracle inequalities. Oracle inequalities
bound the finite-sample distance between the empirically obtained
decision function and that of the omniscient oracle, namely,
the true risk of decision function. Before giving
theoretical results for $f_{D,\lambda}$,
we present the following notation and assumptions.

\begin{assumption}\label{assp:bond_PHat}
The following property of the estimator $\widehat{\pi}(X)$ holds.
\[0<c_{n,\text{L}}\leq\widehat{\pi}(X)\leq c_{n,\text{U}}<1,\]
where $\frac{1}{c_{n,\text{L}}}$ and $\frac{1}{1-c_{n,\text{U}}}$ are $O\left(n^{d}\right)$, $0\leq d <\frac{1}{2}$.
\end{assumption}
Note that this assumption can always be satisfied by taking
\[\widehat{\pi}(X)\equiv\min\left\{\max\left(c_{n,\text{L}},~\widetilde{\pi}(X)\right),~c_{n,\text{U}}\right\},\]
where $\widetilde{\pi}(X)$ is some estimator. Moreover, if a lower bound on the constant $c$ in Assumption~\ref{MAR} is known, then $d\equiv0$.

Let $B_{\mathcal{H}}\equiv\{f\in \mathcal{H}:\|f\|_{\mathcal{H}}\leq 1\}$
be the unit ball in the {RKHS} $\mathcal{H}$.
Define $\mathcal{N}(B_{\mathcal{H}},\|\cdot\|_{\infty},\varepsilon)$
as the $\varepsilon$-covering number of $B_{\mathcal{H}}$ w.r.t.\
the supremum norm $\|\cdot\|_{\infty}$, where
\[\|f(x)\|_{\infty}=\text{ess}\sup \{\left|f(x)\right|, x\in\mathcal{X}\}.\]

In order to show the universal consistency of the two proposed kernel machines,
we need the following two conditions which depend on the choice of kernel and loss function. Both conditions can be verified.
\begin{condition}\label{condi:entroy}
There are constants $a>1$, $p>0$, such that for every $\varepsilon>0$,
the entropy of $B_{\mathcal{H}}$ is bounded as follows
\[\log\mathcal{N}(B_{\mathcal{H}},\|\cdot\|_{\infty},\varepsilon)\leq a\varepsilon^{-2p}.\]
\end{condition}

\begin{condition}\label{condi:LipCons}
There are constants $q>0$ and $r\geq 1$, such that the locally Lipschitz constant
is bounded by $C_{L}(\lambda)\leq r\lambda^{q}$.
\end{condition}

\begin{remark}\label{Con1_Con2}
Condition $\ref{condi:entroy}$ is used to bound the entropy of the function space
$\mathcal{H}$. Linear, Taylor, and Gaussian RBF kernels satisfy for all $p>0$, since all
of them are infinitely often differentiable \citep[Section 6.4]{Steinwart:Christmann:08}.

For the hinge loss, Condition~\ref{condi:LipCons} holds with $q=0$. For the quadratic loss,
Condition~\ref{condi:LipCons} holds with $q=1$ \citep[Section 2.2]{Steinwart:Christmann:08}.
\end{remark}

Define
\begin{eqnarray}\notag
Err_{1,n} &=& \sup_{x\in\mathcal{X}}\left|\widehat{\pi}(x)-\pi(x)\right|,\\\label{eq:Err}
Err_{2,n} &=& \sup_{f\in\mathcal{H}_{n}}\left\|\widehat{H}(x,f(x))-H(x,\beta_0,f(x))\right\|_{\infty},
\end{eqnarray}
as the missing mechanism estimation error, and the conditional risk estimation error, respectively.
Here $\mathcal{H}_n$ is the subspace on which the minimization takes place.

For calculating the learning rates we need the following assumption, which is not needed for the consistency results. Define $f_{\prob,\lambda}=\inf_{f\in \mathcal{H}}\lambda\|f\|^{2}_{\mathcal{H}}+R_{L,\prob}(f)$,
and the approximation error is given by
\[A_2(\lambda)\equiv
\lambda\|f_{\prob,\lambda}\|^{2}_{\mathcal{H}}+R_{L,\prob}(f_{\prob,\lambda})-\inf_{f\in\mathcal{H}}R_{L,\prob}(f).\]
\begin{assumption}\label{assp:BoundA2}
There exist constant $b$ and $\gamma\in(0,1]$ such that
\[A_2(\lambda)\leq b\lambda^{\gamma},\quad \lambda\geq 0.\]
\end{assumption}
This assumption is used to establish learning rates.

\subsection{Theoretical results of weighted-complete-case kernel machines}

Recall that $\mathcal{P}$ is the set of all probability distributions that follow Assumption~\ref{MAR}. We have the following consistency result for the weighted-complete-case kernel machines.

\begin{theorem}\label{theo:Cons_WCC}
Let Assumptions~\ref{MAR} and \ref{assp:bond_PHat} hold. Let the kernel $k$ and the loss function $L$ be such that Conditions~\ref{condi:entroy} and \ref{condi:LipCons} hold. Assume that $\left|\widehat{\pi}(X)-\pi(X)\right|=\O_p\left(n^{-\frac{1}{2}}\right)$. Choose $\lambda_{n}\longrightarrow0$ and \[\lambda_{n}^{\frac{q+1}{2}}n^{\min\left(\frac{1}{2}-d,\frac{1}{2p+2}\right)
-\frac{(q+1)d}{2}}\longrightarrow\infty,
\]
where $0<\lambda_{n}<1$.
Then, the weighted-complete-case kernel machine is $\mathcal{P}$-universally consistent.
In other words, $R_{L,\prob}\left(f_{D,\lambda}^{W}\right)\stackrel{\text{\prob}}
\longrightarrow R_{L,\prob}^{\ast}$ for all $\text{\prob}\in\mathcal{P}$.
\end{theorem}

The proof of Theorem~\ref{theo:Cons_WCC} is based on an oracle inequality derived for weighted-complete-case kernel machines and can be found in the Supplementary Material (see Theorem~\ref{ineq:OIWCC}). Note that when $L$ is the quadratic loss, the kernel $k$ is Gaussian, and $d\equiv 0$, then $\lambda_{n}$ should be chosen such that $\lambda_{n} n^{\frac{1}{2}-\epsilon}\rightarrow \infty$ for an arbitrary small $\epsilon>0$.

Next we derive the learning rate of this learning method.
\begin{corollary}\label{coro:learnRate_WCC}
Let Assumptions~\ref{MAR}, \ref{assp:bond_PHat} and~\ref{assp:BoundA2} hold. Let the kernel $k$ and the loss function $L$ be such that Conditions~\ref{condi:entroy} and \ref{condi:LipCons} hold. Assume that $\left|\widehat{\pi}(X)-\pi(X)\right|=\O_p\left(n^{-\frac{1}{2}}\right)$. Then, the
learning rate of the weighted-complete-case kernel-machine learning method is
\[n^{\left\{-\min\left(\frac{1}{2p+2},\frac{1}{2}-d\right)+\frac{(q+1)d}{2}\right\}\frac{2\gamma}{2\gamma+q+1}}.\]
\end{corollary}
Note that when $L$ is the quadratic loss, the kernel $k$ is Gaussian, and $d\equiv 0$,
the learning rate is $n^{\frac{\gamma}{2\gamma+2}-\epsilon}$ for an arbitrary small $\epsilon>0$.

\subsection{Theoretical results of the doubly-robust kernel machines}
Before giving the theoretical results of the doubly-robust kernel machines,
we discuss some convergence orders related to this learning
method. Let $\mathbb{P}_{n}f=\frac{1}{n}\sum_{i=1}^{n}f(X_i)$ be the empirical measure on
sample value $X_1,\dots,X_n$.
Define
\begin{align*}
a_{n}&\equiv\mathbb{P}_{n}\left[\frac{M-\pi(X)}{\pi(X)}\right]
=\frac{1}{n}\sum_{i=1}^{n}\frac{M_i-\pi(X_i)}{\pi(X_i)},\\
h_{n}&\equiv\sup_{f\in \mathcal{H}_n}\|\mathbb{P}_{n}\left\{L(X,Y,f(X))-H(X,\beta_0,f(X))\right\}\|_{\infty},
\end{align*}
where $\mathcal{H}_n$ is the subspace on which the minimization takes place. Note that $a_n$ is the mean of i.i.d.\ bounded random variables and hence $a_n=\O_p\left(n^{-\frac{1}{2}}\right)$. However, unlike $a_n$, the term $h_n$ is a supremum of a random process over of set of functions $f\in \mathcal{H}_n$, where $\mathcal{H}_n$ is a space that grows with $n$. We discuss the functional space $\mathcal{H}_n$ in the proof of the following lemma.
\begin{lemma}\label{lem:hn}
Let $q$ be the constant in Condition $\ref{condi:LipCons}$. Then
\[h_n= \O_p\left(n^{-\left\{\frac{1}{2}-\frac{(q+1)d}{4}\right\}}\lambda^{-\frac{q+1}{4}}\right),\]
where $d$ appears in Assumption~\ref{assp:bond_PHat}.	
\end{lemma}

\begin{lemma}\label{Bond:Err2}
Assume that $F_{Y|X}(y\mid x,\beta)$ is continuously differentiable with respect to $\beta$ for
$x\in\mathcal{X}$, $y\in\mathcal{Y}$.
Assume that $\left|\widehat{\beta}-\beta_0\right|=\O_p\left(n^{-\frac{1}{2}}\right)$,
then, $Err_{2,n}=\O_p\left(n^{-\frac{1}{2}+\frac{(q+1)d}{2}}\lambda^{-\frac{q+1}{2}}\right)$, where $Err_{2,n}$ appears in~\eqref{eq:Err}.	
\end{lemma}

\begin{remark}\label{re:order_HE2}
Recall that for the quadratic loss, $q=1$. Thus, by Lemmas~\ref{lem:hn} and~\ref{Bond:Err2},
\begin{eqnarray*}
h_n&=& \O_p\left(n^{-\frac{1}{2}+\frac{d}{2}}\lambda^{-\frac{1}{2}}\right),\\
Err_{2,n}&=&\O_p\left(n^{-\frac{1}{2}+d}\lambda^{-1}\right).
\end{eqnarray*}
\end{remark}

The previous three convergence orders of $a_n$,
$h_n$ and $Err_{2,n}$ are used to prove the universal consistency
and derive the learning rate of the doubly-robust kernel
machine learning method.
The following theorem describes the universal consistency property for doubly robust kernel machines.

\begin{theorem}\label{theo:Cons_DR}
Let Assumptions~\ref{MAR} and~\ref{assp:bond_PHat} hold. Let the loss function $L$ be
a quadratic loss. Assume that either $\left|\hat\pi(X)-\pi(X)\right|=\O_p\left(n^{-\frac{1}{2}}\right)$
or $\left|\widehat{\beta}-\beta_0\right|=\O_p\left(n^{-\frac{1}{2}}\right)$. 	
Choose  $0<\lambda_n<1$, such that $\lambda_n\longrightarrow0$ and
\[\lambda_n n^{\min\left(\frac{1}{2p+2},\frac{1}{2}-d\right)-d}\longrightarrow\infty.\]
Then,
$R_{L,\prob}\left(f_{D,\lambda}^{DR}\right)\stackrel{\text{\prob}}\longrightarrow R_{L,\prob}^*$
for all $P\in\mathcal{P}$.
\end{theorem}

The proof of Theorem \ref{theo:Cons_DR} is based on an oracle inequality derived for
doubly-robust kernel machines and can be found in the Supplementary Material (see Theorem~\ref{ineq:OIDR}). Note that
when the kernel $k$ is Gaussian, and $d\equiv0$,
then $\lambda_{n}$ should be chosen such that
$\lambda_{n} n^{\frac{1}{2}-\epsilon}\rightarrow \infty$ for an arbitrary small $\epsilon>0$.

Finally, we derive the learning rate based on previous results.
\begin{corollary}\label{coro:learnRate_DR}
Let Assumptions~\ref{MAR}, \ref{assp:bond_PHat} and \ref{assp:BoundA2} hold.
Let the loss function $L$ be a quadratic loss and the kernel $k$  be such that
Conditions $\ref{condi:entroy}$, $\ref{condi:LipCons}$ hold.
If either $\left|\hat\pi(X)-\pi(X)\right|=\O_p\left(n^{-\frac{1}{2}}\right)$
or $\left|\widehat{\beta}-\beta_0\right|=\O_p\left(n^{-\frac{1}{2}}\right)$, then
the learning rate of the doubly-robust kernel machine learning method is  $n^{\left\{-\min\left(\frac{1}{2p+2},\frac{1}{2}-d\right)+d\right\}\frac{\gamma}{\gamma+1}}$.
\end{corollary}

Note that when the kernel $k$ is Gaussian, and $d\equiv 0$,
the learning rate is $n^{\frac{\gamma}{2\gamma+2}-\epsilon}$, $\epsilon>0$ is arbitrary small.

\section{Simulation study}\label{sec:sim}
We conducted a simulation study to evaluate the finite-sample performance of the
proposed kernel methods for both regression and classification. We compare
the proposed methods with the following three existing methods.
\begin{description}
\item[Reg]  The linear regression method which uses only complete cases.
\item[SSL]  The semi-supervised linear regression method of \cite{Azriel:Brown:16}
which takes into account the missing responses.
\item[CC] The naive kernel machines which use only the complete observations.
\end{description}
For the proposed kernel machines, we consider the following six different settings.
\begin{description}
\item[WCC-M] Weighted-complete-case kernel machines with a misspecified missing mechanism,
 which is estimated by a generalized linear model through the probit link function.
\item[WCC-C] Weighted-complete-case kernel machines with a correctly specified missing mechanism,
namely a generalized linear model with the logit link function.
\item[DR-M] Doubly-robust kernel machines with a misspecified missing mechanism
and a misspecified regression model.
\item[DR-MR] Doubly-robust kernel machines with a misspecified
regression model but a correctly specified missing mechanism.
\item[DR-MM] Doubly-robust kernel machines with a correctly
 specified regression model and a misspecified missing mechanism.
\item[DRC] Doubly-robust kernel machines with a correctly specified regression
 model and a missing mechanism.
\end{description}

We consider four generating data mechanisms. The first example is a toy example
that shows the price of ignoring the missing responses. Specifically,
in this example, both the density and the missing rate are getting larger
with the first covariate $X$. Thus, ignoring the missingness,
yields an estimator which is based on the smaller values of $X$.
However, since the response is a nonlinear curve in $X$,
this may lead to a biased estimation. The model is given by
$Y=\exp(X)+U_2+U_3+U_4+U_5+\varepsilon$ where $X\sim4\cdot \text{Beta}(5,3)$, $U_2,\ldots U_5$
are independent uniform variables on $[0,4]$, and $\varepsilon$ is a standard normal variable.
The missing mechanism is given by
\[\prob\left(M=1\mid X\right)
=\begin{cases}
\left[{1+\exp\left\{\frac{9}{2}(X-2)\right\}}\right]^{-1}& X\leq2,\\
\left[{1+\exp\left\{-(X-4)\right\}}\right]^{-1}& \text{otherwise}.
\end{cases}\]
For $X$ in segment $[0,2]$, the missing rate is about 22\%,
while 77\% for $X$ in segment $(2,4]$. The overall missing rate is about 64\%,
see Figure~\ref{plot:YandR}.
\begin{figure}[!t]
\centering
\includegraphics[width=0.6\textwidth]{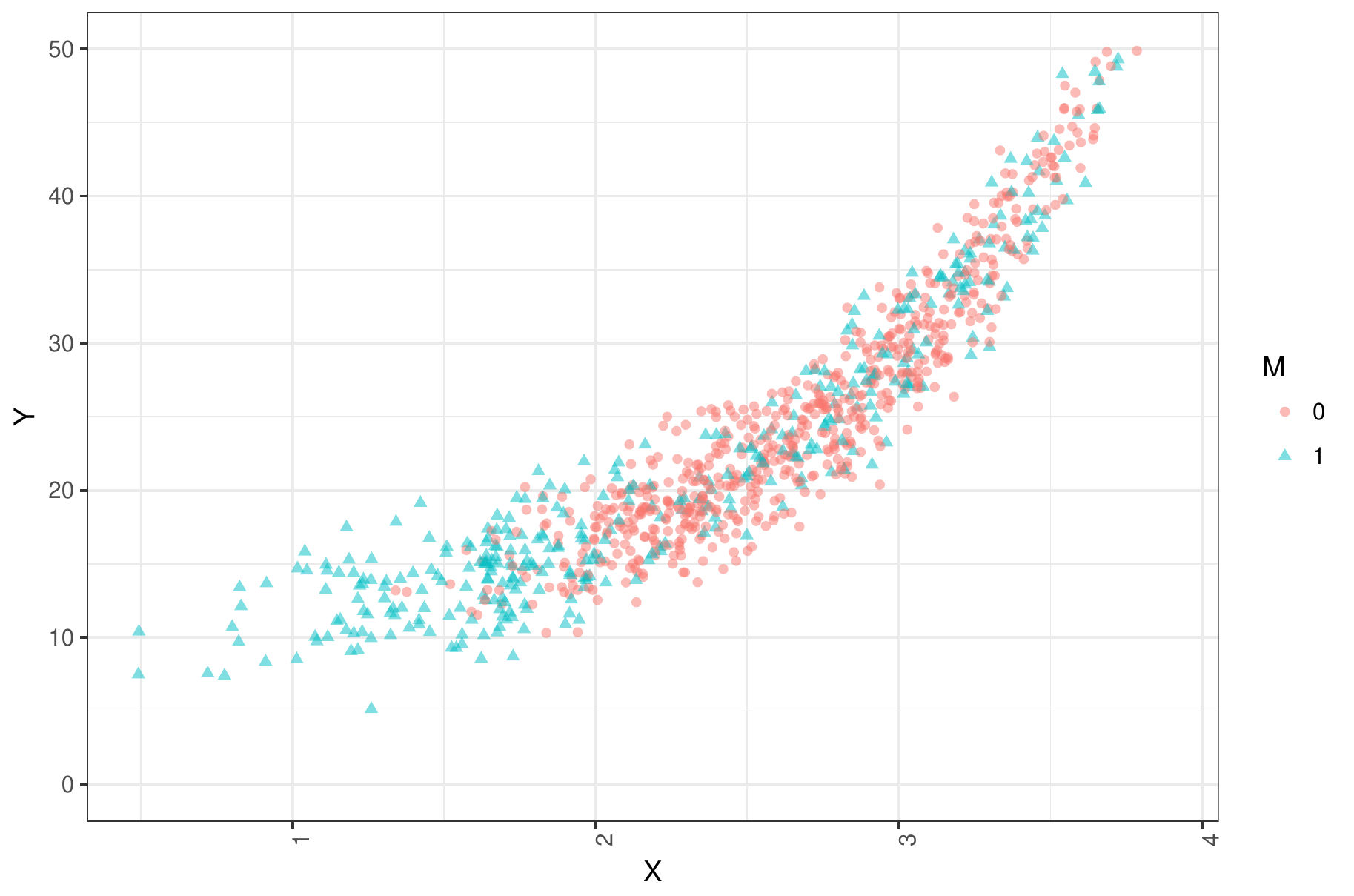}
\caption{\scriptsize{Plot of Setting~1 with sample size 1000: The blue triangle points are observed responses and
the red circle points are missing responses. Observations with $X$ on [0,2]
have lower probability to be generated than $X$ on (2,4]
while more easily to be observed.}\label{plot:YandR}}
\end{figure}

Setting~2 is a classification setting which is considered by \cite{Laber:Murphy11}. Data are
generated as $Y=\text{sign}\left(X_2-\frac{4}{25}X_1^{2}-1+\varepsilon\right)$, where $X_1$
and $X_2$ are independent uniform random variables on $[0,5]$, $\varepsilon$ is a normal variable
with mean 0 and standard deviation $\frac{1}{2}$.
The missing mechanism is
$\prob\left(M=1\mid X\right)
=\frac{\exp\left\{\frac{3}{2}(X_2-X_1)\right\}}{1+\exp\left\{\frac{3}{2}(X_2-X_1)\right\}}$.
For $Y=1$, the missing rate is about 20\% while 84\% for $Y=-1$, which means that positive labels are
more easily observed. The overall missing rate is about 50\%.

The last two settings are taken from examples in \cite{Liu:Lin:07}.
These two settings are motivated from prostate-specific antigen (PSA) which is routinely
used as a biomarker for prostate cancer screening. \cite{Liu:Lin:07} studied the genetic
pathway effect on PSA and use least-squares kernel
machines to model the genetic pathway effect. Consider a generic regression model
$Y=Z+h\left(X_1,\ldots,X_p\right)+\varepsilon$, where
$X_1,\ldots X_p$ are independent uniform variables on $[0,1]$,
$Z=3\cos(X_1)+2U$, where $U$ is also a uniform random variable on $[0,1]$, $h(\cdot)$
is a centered smooth function, and $\varepsilon$ is an independent standard normal random variable.
In Setting~3, $p=5$, and
\begin{align*}
&h\left(X_1,\ldots,X_5\right)\\
&=10\cos(X_1)-15X_2^{2}+
10\exp(-X_3)Z_4-8\sin(X_5)\cos(X_3)+20X_1X_5.
\end{align*}
The missing mechanism is given by
\[\prob\left(M=1\mid X\right)=
\frac{\exp\left\{-\frac{4}{3}\log3+\frac{2}{3}\log3\sum_{i=1}^{5}\frac{X_i}{5}\right\}}
{1+\exp\left\{-\frac{4}{3}\log3+\frac{2}{3}\log3\sum_{i=1}^{5}\frac{X_i}{5}\right\}}.\]
In Setting~4, $p=10$, and
\begin{align*}
&h\left(X_1,\ldots,X_{10}\right)\\
&=10\cos(X_1)-15X_2^{2}+10\exp(-X_3)X_4-8\sin(X_5)\cos(X_3) \\
&\quad+20X_1X_5+9X_6\sin(X_7)-8\cos(X_6)X_7+20X_8\sin(X_9)
\sin(X_{10})\\
&\quad-15X_{8}^{3}-10X_8X_9-\exp(X_{10})\cos(X_{10}).
\end{align*}
The missing mechanism is given by
\[\prob\left(M=1\mid X\right)=
\frac{\exp\left\{-\frac{4}{3}\log3+\frac{2}{3}\log3\sum_{i=1}^{10}\frac{X_i}{10}\right\}}
{1+\exp\left\{-\frac{4}{3}\log3+\frac{2}{3}\log3\sum_{i=1}^{10}\frac{X_i}{10}\right\}}.\]

In Setting~3 and Setting~4, the missing rate is about 75\%.
We repeated the simulations 100 times for each of the sample sizes 100, 200,
400, and 800. We used a testing dataset of size 100,000 to evaluate the performance
of different methods. The results are summarized in Table~\ref{tab:4settings} and Figure~\ref{plot:4boxplot}. A link
to the code for both the algorithm and the simulations can be found
in Supplementary Material.

\begin{figure}[!t]
\centering
\includegraphics[width=0.63\textwidth]{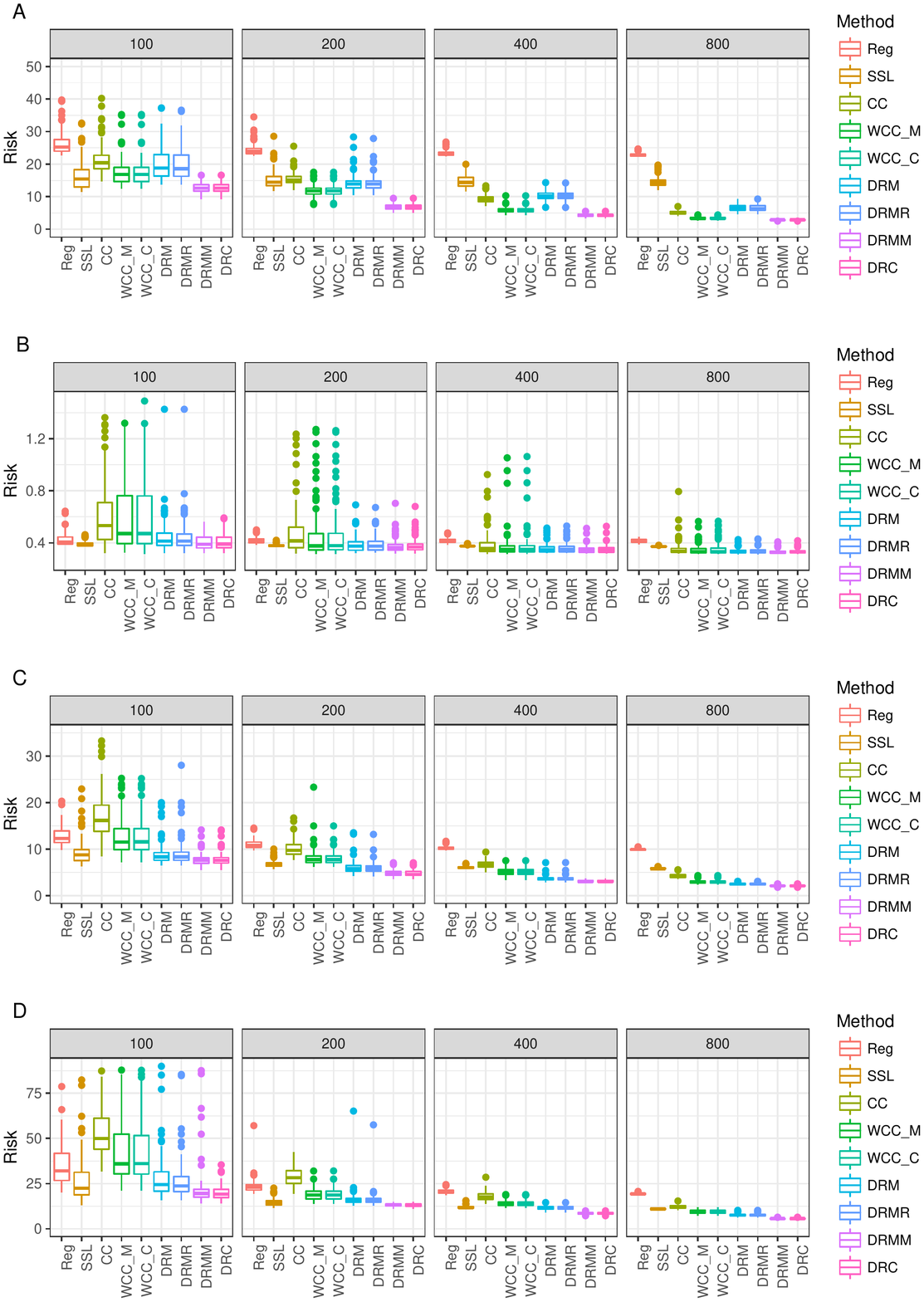}
\caption{\scriptsize{Boxplots of the four settings: rows A-D show the results
for Settings~1-4 respectively.  The four settings compare risks of
nine different methods.}\label{plot:4boxplot}}
\end{figure}

To summarize, the simulations show that the doubly-robust
kernel machine methods perform better, in general, than the
other existing methods. When the
regression model for doubly-robust kernel machines is correctly specified, the doubly-robust
kernel machine methods are recommended. Additionally, if the missing mechanism
is correctly estimated, the doubly-robust estimator is the best choice.
When little information about the regression model is given, the weighted-complete-case kernel machine
is another good choice, especially for large sample size datasets.

\section{Application of Los Angeles homeless population}\label{sec:ReaDat}

We applied the proposed kernel-machine methods to the Los Angeles homeless dataset.
The dataset is described in \cite{Kriegler:Berk:10} and  \cite{Azriel:Brown:16}.
The dataset contains information about 2054 census tracts in the Los Angeles county,
where the goal is to estimate the number of homeless in each tract.
Due to budget limitation, some tracts were not visited, and consequently,
the number of homeless in these tracts is missing. The missing mechanism
depends on the Service Provision Area (SPA) to which the tract belongs.
We use this dataset to compare the performance of the methods mentioned above.

Following \cite{Azriel:Brown:16}, we first delete all tracts with zero median household income and the highly-populated tracts leaving 1797 tracts in the dataset.
We also used the same covariate sets as in \cite{Azriel:Brown:16}.
To evaluate the performance of different methods, we randomly choose
1597 tracts to train the algorithms and then use the 200 tracts to test.
The risk is calculated by the mean square error (MSE) and the weighted mean square error,
where the weights are the inverse probability of the tract to be visited.
We repeated  this process 100 times.

Since the data are skewed, we first took the log transformation
of the observed response and the covariates ``Industrial", ``PctVacant",
``Commercial", and ``MedianHouseIncome". No transformation was done for the covariates  ``Residential" and ``PctMinority". We normalized the data.
Boxplots of the data after transformation and normalization  are shown in Figure~\ref{plot:LA_Var_boxplot}.
\begin{figure}[!t]
\centering
\includegraphics[width=0.5\textwidth]{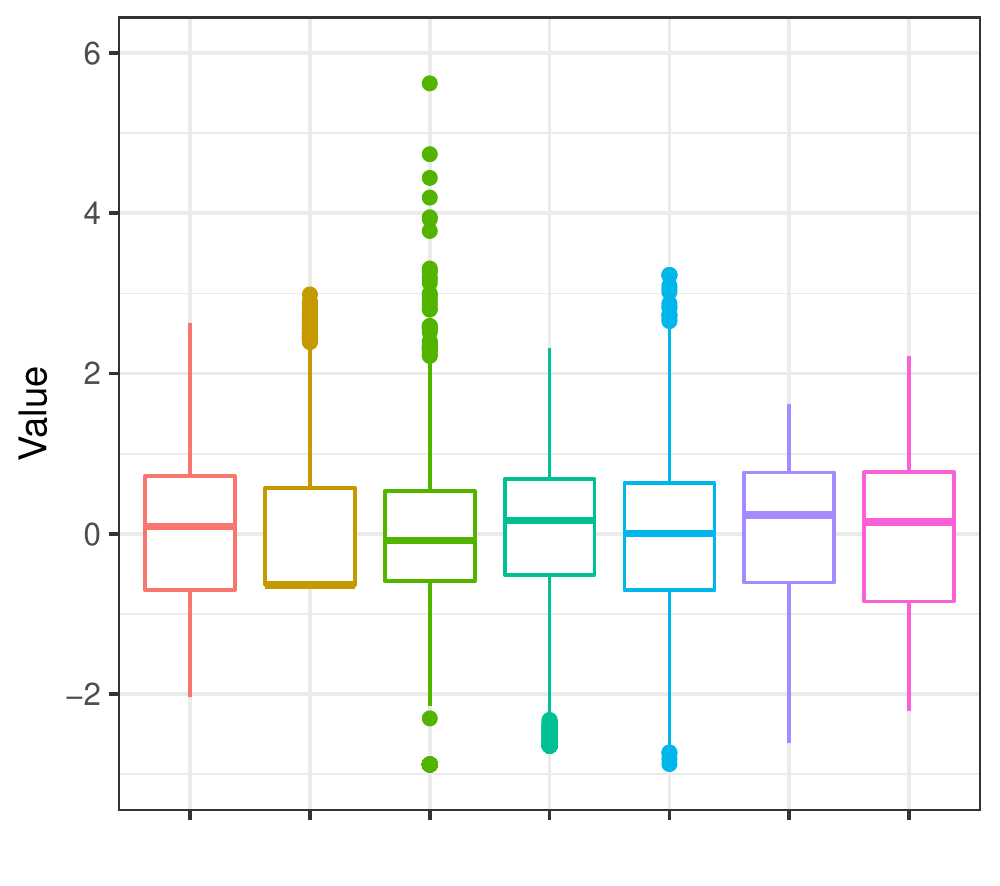}
\caption{\scriptsize{From left to right, boxplots of the observed response, and the covariates Industrial, PctVacant, Commercial,and Median House Income, after log transformation and normalization. The last two boxplots are of the covariates Residential and PctMinority after normalization}\label{plot:LA_Var_boxplot}}
\end{figure}

We considered the previous \textbf{Reg}, \textbf{SSL}, \textbf{CC}, \textbf{WCC}
and \textbf{DR} methods. For the kernel machine methods we used the RBF kernel.
Particularly, we used the semi-supervised linear regression
method to estimate the responses in the doubly-robust kernel machine method.
Table~\ref{tab:LA_mean_std} provides the numerical results of the five different methods.
Overall the methods perform similarly while
the weighted-complete-case kernel machines perform best for both
the MSE and the weighed MSE with the lowest mean, median and standard
deviation of the two kinds of risk. In this example, the doubly-robust
kernel machine does not perform well and this could be related to the performance
of the semi-supervised linear regression
method used in the augmentation term.
\begin{table}[H]
\centering
\tiny
\caption{\scriptsize{Descriptive statistics: median, mean, and standard deviation}\label{tab:LA_mean_std}}
\begin{tabular}{lccccccc}
\toprule
        &              &Method& Reg    & SSL      & CC      & WCC & DR\\
MSE     &              &     &         &          &         &      &        \\
\midrule
mean    & weighted     &     & 10103   & 10379    & 9971    & 9661 & 10158   \\
        & not weighted &     & 1103    & 1127     & 1086    & 1054 & 1092    \\
        &              &     &         &          &         &      &         \\
median  & weighted     &     & 8668    & 8478     & 9069    & 8470 & 8944    \\
        & not weighted &     & 889     & 880      & 942     & 894  & 947     \\
        &              &     &         &          &         &      &         \\
std     & weighted     &     & 7371    & 7671     & 7343    & 7232 & 7475    \\
        & not weighted &     & 852     & 865      & 866     & 849  & 877     \\
\bottomrule
\end{tabular}
\end{table}

\section{Conclusion and discussion}\label{sec:conDis}

We proposed two kernel-machine methods for handling the missing-response problem.
Specifically, we proposed  an inverse-probability complete-case estimator which
can be applied to any convex loss function. We also proposed a quadratic loss based doubly-robust estimator. The empirical risk of these new data-dependent loss functions were shown to be consistent for any function
$h\in\mathcal{H} $ under mild conditions. We presented oracle inequalities and consistency results for both types of kernel machines. We also presented a simulation study and applied these new methods to the Los Angeles homeless dataset.

Several open questions remain and many possible generalizations still exist, especially for
the doubly-robust estimator. We would like to extend the quadratic-loss based doubly-robust
estimator to include other convex loss functions. Additionally, we would like to develop
a new data-dependent loss function for handling missing covariates and guarantee
the doubly-robust property at the same time. This work is under progress were we use imputation methods
to define the augmentation term of a doubly-robust estimator.

\begin{supplement}

	\sname{Supplement Material}\label{suppB}
	\stitle{Code}
    \slink[url]{https://goo.gl/yQUBtA}
	\sdescription{Matlab code for the doubly-robust kernel machine estimator and
weighted-complete-case kernel machine estimator.}
	
\end{supplement}
\appendix

\section{Calculation in Subsection~3.4}
\subsection{Weighted-complete-case kernel machines}
Let
 \begin{align*}
   &g(\valpha)\\
   &=x\sum_{i=1}^{n}\sum_{j=1}^{n}\frac{M_i\left(Y_i-\sum_{j=1}^{n}\alpha_jk(X_i,X_j)\right)^{2}}
   {\widehat{\pi}(X_i)}
   +\lambda\sum_{i=1}^{n}\sum_{j=1}^{n}\alpha_i\alpha_jk(X_i,X_j)\\
   &=(AY-AK\valpha)^{\intercal}(A{\bf Y}-AK\valpha)+\lambda\valpha^{\intercal}K\alpha\\
   &=Y^{\intercal}W{\bf Y}-2\valpha^{\intercal}
   K^{\intercal}W{\bf Y}+\valpha^{\intercal}K^{\intercal}WK\valpha+\lambda\valpha^{\intercal}K\valpha,
 \end{align*}
 where $\valpha$, $A$ and $W$ are defined as in Section~$\ref{sec:km}$.

Taking the derivative and equating to zero,
we have $\widehat{\valpha}=(\lambda I+WK)^{-1}W{\bf Y}$.

\subsection{Doubly-robust kernel machines}
For regression,  we have
\begin{align*}
&\widehat{H}(X_i,f(X_i))\\
&=\left(\mu\left(X_i,\widehat{\beta}\right)-\sum_{j=1}^{n}\alpha_jk(X_i,X_j)\right)^{2}+
\frac{1}{\#\{M=1\}}\sum_{j=1}^{n}M_j\left(Y_{j}-\mu\left(X_j,\widehat{\beta}\right)\right)^{2}.
\end{align*}
Let
\begin{align*}
g(\valpha)
&=\sum_{i=1}^{n}\left[\frac{M_i\left(Y_i-\sum_{j=1}^{n}\alpha_jk(X_i,X_j)\right)^{2}}
{\widehat{\pi}(X_i)}
-\frac{M_i-\widehat{\pi}(X_i)}{\widehat{\pi}(X_i)}\widehat{H}(X_i,f(X_i))\right]\\
&\quad+\lambda\sum_{i=1}^{n}\sum_{j=1}^{n}\alpha_i\alpha_jk(X_i,X_j).
\end{align*}
Note that $\frac{1}{\#\{M=1\}}\sum_{j=1}^{n}M_j\left(Y_{j}-\mu\left(X_j,\widehat{\beta}\right)\right)^{2}$
does not depend on $\valpha$. Consequently,
$g(\valpha)$ has the same minimizer of
\begin{align*}
\widetilde{g}(\valpha)
&=\sum_{i=1}^{n}\left[\frac{M_i\left(Y_i-\sum_{j=1}^{n}\alpha_jk(X_i,X_j)\right)^{2}}{\widehat{\pi}(X_i)}
-\frac{M_i-\widehat{\pi}(X_i)}{\widehat{\pi}(X_i)}\widetilde{H}(X_i,f(X_i))\right]\\
&\quad+\lambda\sum_{i=1}^{n}\sum_{j=1}^{n}\alpha_i\alpha_jk(X_i,X_j)\\
&={\bf Y}^{\intercal}W_1{\bf Y}-2\valpha^{\intercal}K^{\intercal}W_1{\bf Y}
+\valpha^{\intercal}K^{\intercal}W_1K\valpha
  +\vmu\left(X,\widehat{\beta}\right)^{\intercal}W_2\vmu\left(X,\widehat{\beta}\right)\\
&\quad-2\valpha^{\intercal}K^{\intercal}W_2\vmu\left(X,\widehat{\beta}\right)
+\valpha^{\intercal}K^{\intercal}W_2K\valpha+\lambda\valpha^{\intercal}K^{\intercal}\valpha,\\
\end{align*}
where $\vmu\left(X,\widehat{\beta}\right)$, $W_1$ and $W_2$ are defined as in Subsection~\ref{sec:km},
\[\widetilde{H}(X_i,f(X_i))
=\left(\mu\left(X_i,\widehat{\beta}\right)-\sum_{j=1}^{n}\alpha_jk(X_i,X_j)\right)^{2}.\]

Taking the derivative and equating to zero,
\[\widehat{\valpha}=(K+\lambda I)^{-1}\left(W_1{\bf Y}+W_2\vmu\left(X,\widehat{\beta}\right)\right).\]

Next we will derive the kernel machines for the classification problem. Recall that in classification
\begin{eqnarray*}
\widehat{H}(X,f(X))&=&1+f(X)^{2}+2f(X)-4f(X)\text{logit}\left(X,\widehat{\beta}\right)\\
&=&\left[f(X)-\left\{2\text{logit}\left(X,\widehat{\beta}\right)-1\right\}\right]^{2}
+1-\left\{2\text{logit}\left(X,\widehat{\beta}\right)-1\right\}^{2}.
\end{eqnarray*}
Let
\[\widetilde{H}(X,f(X))=\left[f(X)-\left\{2\text{logit}\left(X,\widehat{\beta}\right)-1\right\}\right]^{2}.\]

In this situation, we need to minimize
\begin{align*}
 g(\valpha)
 &=\sum_{i=1}^{n}\left[\frac{M_i\left(Y_i-\sum_{j=1}^{n}\alpha_jk(X_i,X_j)\right)^{2}}
 {\widehat{\pi}(X_i)}
 -\frac{M_i-\widehat{\pi}(X_i)}{\widehat{\pi}(X_i)}\widetilde{H}(X_i,f(X_i))\right]\\
&\quad+\lambda\sum_{i=1}^{n}\sum_{j=1}^{n}\alpha_i\alpha_jk(X_i,X_j).
\end{align*}
Using the same technique as previously,
\[\widehat{\valpha}=(K+\lambda I)^{-1}\left\{W_1{\bf Y}
+W_2\vmu\left(X,\widehat{\beta}\right)\right\},\]
where $\mu\left(X,\widehat{\beta}\right)=2\text{logit}\left(X,\widehat{\beta}\right)-1$.

\section{Tables of simulations}

\begin{table}[H]
\tiny
\caption{\scriptsize{Descriptive statistics of the four settings:
The median, mean, and standard deviation.}\label{tab:4settings}}                                                                \begin{tabular}{c@{\extracolsep{\fill}}ccccccccccccc}
\toprule
Setting&   n       &         & Reg  & SSL  & CC   &WCC\_M&WCC\_C& DRM & DRMR  & DRMM & DRC \\
\midrule
     1 &          & median  & 25.32 & 15.41 & 20.29 & 16.84 & 16.83 & 18.94 & 18.58 & 12.66 & 12.66 \\
       &    100   & mean    & 26.53 & 16.45 & 21.53 & 17.64 & 17.64 & 20.17 & 19.94 & 12.75 & 12.74 \\
       &          & std     & 3.73 & 4.40 & 4.65 & 4.33 & 4.32 & 5.05 & 4.91 & 1.46 & 1.46 \\
       &         &         &       &      &      &      &      &      &      &      &   \\
       &         & median  & 23.94 & 14.45 & 15.13 & 11.81 & 11.80 & 13.91 & 13.87 & 6.70 & 6.69 \\
       &    200  & mean    & 24.51 & 15.26 & 15.42 & 11.72 & 11.71 & 14.34 & 14.24 & 6.82 & 6.81 \\
       &         & std     & 1.90 & 2.79 & 1.85 & 1.78 & 1.78 & 2.71 & 2.54 & 0.92 & 0.91 \\
       &         &         &  &      &      &      &      &      &      &      &   \\
       &         & median  &23.08 & 14.37 & 9.13 & 5.77 & 5.76 & 10.24 & 10.20 & 4.27 & 4.27 \\
       &    400  & mean    &23.33 & 14.76 & 9.25 & 5.92 & 5.91 & 10.31 & 10.26 & 4.31 & 4.31 \\
       &         & std     &0.86 & 1.97 & 1.24 & 0.94 & 0.93 & 1.37 & 1.35 & 0.47 & 0.47 \\
       &         &         &  &      &      &      &      &      &      &      &   \\
       &         & median  & 22.82 & 14.24 & 5.02 & 3.31 & 3.30 & 6.42 & 6.34 & 2.88 & 2.88 \\
       &     800 & mean    & 22.96 & 14.48 & 5.11 & 3.38 & 3.37 & 6.58 & 6.53 & 2.88 & 2.88 \\
       &         & std     & 0.56 & 1.48 & 0.58 & 0.34 & 0.34 & 1.05 & 1.04 & 0.15 & 0.15 \\
\midrule
     2 &         & median  &0.41 & 0.39 & 0.54 & 0.48 & 0.48 & 0.41 & 0.41 & 0.39 & 0.39 \\
       &     100 & mean    &0.43 & 0.39 & 0.66 & 0.62 & 0.61 & 0.45 & 0.45 & 0.42 & 0.42 \\
       &         & std     &0.05 & 0.02 & 0.35 & 0.32 & 0.31 & 0.13 & 0.13 & 0.14 & 0.14 \\
       &         &         &     &      &      &      &      &      &      &      &   \\
       &         & median  &0.41 & 0.38 & 0.42 & 0.38 & 0.38 & 0.37 & 0.38 & 0.36 & 0.37 \\
       &     200 & mean    &0.42 & 0.38 & 0.48 & 0.46 & 0.46 & 0.39 & 0.39 & 0.38 & 0.38 \\
       &         & std     &0.03 & 0.01 & 0.19 & 0.21 & 0.21 & 0.06 & 0.06 & 0.06 & 0.06 \\
       &         &         &     &      &      &      &      &      &      &      &   \\
       &         & median  &0.41 & 0.37 & 0.36 & 0.35 & 0.35 & 0.35 & 0.35 & 0.34 & 0.34 \\
       &     400 & mean    &0.42 & 0.38 & 0.39 & 0.38 & 0.38 & 0.36 & 0.36 & 0.35 & 0.35 \\
       &         & std     &0.02 & 0.00 & 0.10 & 0.11 & 0.12 & 0.04 & 0.04 & 0.03 & 0.03 \\
       &         &         &     &      &      &      &      &      &      &      &   \\
       &         & median  &0.41 & 0.37 & 0.34 & 0.33 & 0.33 & 0.33 & 0.33 & 0.33 & 0.33 \\
       &     800 & mean    &0.41 & 0.37 & 0.35 & 0.35 & 0.35 & 0.34 & 0.34 & 0.33 & 0.33 \\
       &         & std     &0.01 & 0.00 & 0.06 & 0.04 & 0.05 & 0.02 & 0.02 & 0.02 & 0.02 \\
\midrule
   3   &         & median  & 12.31 & 8.77 & 16.17 & 11.70 & 11.68 & 8.34 & 8.34 & 7.59 & 7.58 \\
       &     100 & mean    & 12.84 & 9.29 & 17.04 & 13.74 & 13.30 & 9.56 & 9.10 & 8.04 & 7.80 \\
       &         & std     &  2.04 & 2.90 & 4.94 & 7.20 & 5.75 & 5.92 & 3.00 & 3.51 & 1.41 \\
       &         &         &     &      &      &      &      &      &      &      &   \\
       &         & median  & 10.82 & 6.59 & 9.76  & 7.72 & 7.72 & 5.78 & 5.77 & 4.58 & 4.58 \\
       &     200 & mean    & 10.98 & 6.78 & 10.09 & 8.16 & 8.04 & 5.98 & 5.91 & 4.83 & 4.82 \\
       &         & std     & 0.87  & 0.75 & 1.68  & 2.01 & 1.37 & 1.41 & 1.18 & 0.71 & 0.70 \\
       &         &         &  &      &      &      &      &      &      &      &   \\
       &         & median  &10.21  & 6.03 & 6.62 & 5.10 & 5.10 & 3.60 & 3.60 & 3.03 & 3.03 \\
       &     400 & mean    & 10.28 & 6.07 & 6.64 & 5.12 & 5.12 & 3.64 & 3.64 & 3.09 & 3.09 \\
       &         & std     & 0.38  & 0.26 & 0.82 & 0.82 & 0.82 & 0.52 & 0.52 & 0.27 & 0.27 \\
       &         &        &  &      &      &      &      &      &      &      &   \\
       &         & median &9.93 & 5.78 & 4.21 & 2.95 & 2.95 & 2.46 & 2.47 & 2.09 & 2.09 \\
       &      800& mean   &9.95 & 5.83 & 4.22 & 2.95 & 2.95 & 2.49 & 2.48 & 2.11 & 2.11 \\
       &         & std    &0.15 & 0.14 & 0.42 & 0.35 & 0.35 & 0.19 & 0.19 & 0.13 & 0.13 \\
\midrule
   4   &          & median  & 32.26 & 22.60 & 50.50 & 39.03 & 38.36 & 24.78 & 23.88 & 19.77 & 19.36 \\
       &    100   & mean    & 37.98 & 34.20 & 53.92 & 49.04 & 46.95 & 32.02 & 29.90 & 24.24 & 22.52 \\
       &          & std     & 22.87 & 68.18 & 14.35 & 27.48 & 24.46 & 20.92 & 21.11 & 17.38 & 15.09 \\
       &          &         &  &      &      &      &      &      &      &      &   \\
       &          & median  &23.04 & 14.10 & 28.23 & 18.68 & 18.68 & 15.52 & 15.46 & 13.07 & 13.08 \\
       &     200  & mean    &23.65 & 14.58 & 28.73 & 19.24 & 19.24 & 16.37 & 16.14 & 13.15 & 13.15 \\
       &          & std     &4.15 & 2.13 & 4.74 & 3.56 & 3.57 & 5.23 & 4.42 & 0.91 & 0.92 \\
       &          &         &  &      &      &      &      &      &      &      &   \\
       &          & median  &20.32 & 11.72 & 17.53 & 13.65 & 13.65 & 11.66 & 11.66 & 8.55 & 8.56 \\
       &      400 & mean    & 20.73 & 11.94 & 17.86 & 13.98 & 13.98 & 11.63 & 11.63 & 8.58 & 8.58 \\
       &          & std     & 1.28 & 0.86 & 2.47 & 1.56 & 1.56 & 1.14 & 1.14 & 0.48 & 0.48 \\
       &          &        &  &      &      &      &      &      &      &      &   \\
       &          & median  &19.28 & 10.95 & 12.22 & 9.16 & 9.16 & 7.52 & 7.51 & 5.68 & 5.69 \\
       &     800  & mean    &19.33 & 11.00 & 12.23 & 9.35 & 9.35 & 7.58 & 7.58 & 5.68 & 5.68 \\
       &          & std     &0.37 & 0.27 & 0.97 & 1.07 & 1.07 & 0.66 & 0.65 & 0.26 & 0.26 \\
\bottomrule
\end{tabular}
\end{table}

\section{Proofs of theoretical results}
\subsection{Proof of Lemma $\ref{lemma3}$}
\begin{proof}
Let $L(y,t)=(y-t)^{2}$. We now prove that
$L_{\widehat{W},\widehat{H}}\left(\widehat{\pi},\widehat{H},M,X,Y,f(X)\right)$
is convex.  The same argument can be used for $L_{W,H}$.

Recall that
\[\widehat{H}(X,t)=\int_{y\in\mathcal{Y}} L(y,t)dF_{Y|X}\left(y\mid X,\widehat{\beta}\right).\]
We first show that for every convex loss $L$, $\widehat{H}(X,t)$ is convex.
For any $\alpha\in(0,1)$, by the convexity of $L(y,t)$ with respect to $t$,
\begin{align*}
&\widehat{H}\left(X,\alpha t+(1-\alpha)t'\right) \\
& =\int_{y\in\mathcal{Y}} L\{y,\alpha t+(1-\alpha)t'\}dF_{Y|X}\left(y\mid X,\widehat{\beta}\right)\\
 &\leq \int_{y\in\mathcal{Y}}\left\{\alpha L(y,t)+(1-\alpha)L(y,t')\right\}dF_{Y|X}\left(y\mid X,\widehat{\beta}\right)\\
 &=\alpha\int_{y\in\mathcal{Y}}L(y,t)dF_{Y|X}\left(y\mid X,\widehat{\beta}\right)
 +(1-\alpha)\int_{y\in\mathcal{Y}}L(y,t')dF_{Y|X}\left(y\mid X,\widehat{\beta}\right)\\
 &=\alpha\widehat{H}(X,t)+(1-\alpha)\widehat{H}(X,t'),
\end{align*}
which indicates that $\widehat{H}(X,t)$ is a convex function with respect to $t$.

Therefore, when $M=0$, $L_{\widehat{W},\widehat{H}}\left(\widehat{\pi},\widehat{H},0,X,Y,t\right)
=\widehat{H}(X,t)$ which is a convex function for any loss $L$.
When $M=1$ and $L$ is the quadratic loss,

\begin{align*}
\widehat{H}(X,t) &=\int_{y\in\mathcal{Y}}\left(y-t\right)^{2}dF_{Y|X}\left(y\mid X,\widehat{\beta}\right)
=t^{2}-2tU\left(X,\widehat{\beta}\right)+V\left(X,\widehat{\beta}\right),
\end{align*}
where $U\left(X,\widehat{\beta}\right)=\int_{y\in\mathcal{Y}}ydF_{Y|X}\left(y\mid X,\widehat{\beta}\right)$ and
$V\left(X,\widehat{\beta}\right)=\int_{y\in\mathcal{Y}}y^{2}dF_{Y|X}\left(y\mid X,\widehat{\beta}\right)$.
Note that $U$ and $V$ are not functions of $t$. Hence, for $M=1$,

\[L_{\widehat{W},\widehat{H}}\left(\widehat{\pi},\widehat{H},M,X,Y,t\right)
=t^{2}-2t\left\{\frac{Y}{\widehat{\pi}(X)}-\frac{1-\widehat{\pi}(X)}{\widehat{\pi}(X)}U\left(X,\widehat{\beta}\right)\right\}
+\frac{Y^{2}-\left\{1-\widehat{\pi}(X)\right\}V\left(X,\widehat{\beta}\right)}{\widehat{\pi}(X)}.\]
Since the second derivative with respect to $t$ is positive,
$L_{\widehat{W},\widehat{H}}\left(\widehat{\pi},\widehat{H},M,X,Y,t\right)$ is convex with respect to $t$.
\end{proof}

\subsection{Proof of Lemma $\ref{lemma2}$}
\begin{proof}

Case 1:
The missing mechanism is correctly specified, that is
$\widehat{\pi}(X)\stackrel{\text{P}}\longrightarrow\pi^{\ast}(X)=\pi(X)$,
but $\widehat\beta\stackrel{\text{P}}\longrightarrow \beta^{\ast}$,
where $\beta^{\ast}$ does not necessarily equal $\beta_0$.
\begin{align*}
  &L_{\widehat{W},\widehat{H}}\left(\widehat{\pi},\widehat{H},M,X,Y,f(X)\right)\\
  &= \frac{ML(Y,f(X))}{\pi(X)}-\frac{M-\pi(X)}{\pi(X)}H(X,\beta^{\ast},f(X))+\o_{p}(1) \\
  &= L(Y,f(X))+\frac{M-\pi(X)}{\pi(X)}\left\{L(Y,f(X))-H(X,\beta^{\ast},f(X))\right\}+\o_{p}(1).
\end{align*}
By the Law of Large Number (LLN), we have
\[R_{L_{\widehat{W},\widehat{H}},D}(f)\stackrel{\text{P}}\longrightarrow R_{L,\prob}+\E\left[\frac{M-\pi(X)}{\pi(X)}\left\{L(Y,f(X))-H(X,\beta^{\ast},f(X))\right\}\right].\]
Note that
\begin{align*}
   & \E\left[\frac{M-\pi(X)}{\pi(X)}\left\{L(Y,f(X))-H(X,\beta^{\ast},f(X))\right\}\right] \\
   &=\E\left(\E\left[\frac{M-\pi(X)}{\pi(X)}\left\{L(Y,f(X))-H(X,\beta^{\ast},f(X))
   \right\}\Big|X,Y\right]\right)\\
   & =\E\left(\left\{L(Y,f(X))-H(X,\beta^{\ast},f(X))\right\}
   \E\left[\frac{M-\pi(X)}{\pi(X)}\Big| X,Y\right]\right)\\
   & =\E\left(\left\{L(Y,f(X))-H(X,\beta^{\ast},f(X))\right\}
   \E\left[\frac{M-\pi(X)}{\pi(X)}\Big| X\right]\right)
   =0.
\end{align*}
The third equality holds because $M$ and $Y$ are independent given $X$. As
a conclusion, we have,
$ R_{L_{\widehat{W},\widehat{H}},D}(f)\allowbreak\stackrel{\text{P}}\longrightarrow R_{L,\prob}(f)$.

Case 2: When $\widehat\beta\stackrel{\text{P}}\longrightarrow \beta^{\ast}=\beta_0$,
but $\widehat{\pi}(X)\stackrel{\text{P}}\longrightarrow\pi^{\ast}(X)$ which
is not necessarily $\pi(X)$,
\[L_{\widehat{W},\widehat{H}}\left(\widehat{\pi},\widehat{H},M,X,Y,f(X)\right)
 =H(X,\beta_0,f(X))+\frac{M\left\{L(Y,f(X))
 -H(X,\beta_0,f(X))\right\}}{\pi^{\ast}(X)}+\o_{p}(1).\]
Then
\[R_{L_{\widehat{W},\widehat{H}},D}(f)\stackrel{\text{P}}\longrightarrow \E\{H(X,\beta_0,f(X))\}+\E\left[\frac{M\left\{L(Y,f(X))
 -H(X,\beta_0,f(X))\right\}}{\pi^{\ast}(X)}\right].\]
The second expression can be shown equal to 0. Indeed,
\begin{align*}
 &\E\left[\frac{M\left\{L(Y,f(X))-H(X,\beta_0,f(X))\right\}}{\pi^{\ast}(X)}\right] \\
 &=\E\left(\E\left[\frac{M\left\{L(Y,f(X))-H(X,\beta_0,f(X))\right\}}
 {\pi^{\ast}(X)}\bigg|M,X\right]\right)\\
 &=\E\left(\frac{M}{\pi^{\ast}(X)}
 \E\left[\left\{L(Y,f(X))-H(X,\beta_0,f(X))\right\}\mid X\right]\right)\\
 &=\E\left[\frac{M}{\pi^{\ast}(X)}\E\left\{L(Y,f(X))\mid X\right\}-H\left(X,\beta_0,f(X)\right)\right]=0,
\end{align*}
where the last equation holds since by $\eqref{eq:CondiExpct}$, $H\left(X,\beta_0,f(X)\right)$ is defined as
$\E\left\{L(Y,f(X))\mid X\right\}$.
Note that $H(X,\beta_0,f(X))=R_{L,\prob}(f)$ and the result follows.
\end{proof}

\subsection{Oracle Inequality for Weighted-Complete-Case Kernel Machines}
\begin{theorem}\label{ineq:OIWCC}
Let Assumptions~\ref{MAR} and \ref{assp:bond_PHat} hold.
Then, for fixed $\lambda > 0$, $n \geq 1$, $\varepsilon > 0$, and $\eta > 0$,
with probability not less than $1-e^{-\eta}$,
\begin{align*}
&\lambda\|f_{D,\lambda}\|^{2}_{\mathcal{H}}+R_{L,\prob}(f_{D,\lambda})-\inf_{f\in \mathcal{H}}R_{L,\prob}(f) \\
 &<A_2(\lambda)+d_{2n}(\lambda)\varepsilon+d_{3n}(\lambda)
 \left[\sqrt{\frac{2\eta+2\log\left\{2\mathcal{N}\left(B_{\mathcal{H}},\|\cdot\|_{\infty},d_{1n}(\lambda)\varepsilon\right)\right\}}{n}}+
 \frac{Err_{1,n}}{c_{n,L}}\right],
\end{align*}
\[\text{where}~d_{1n}=(c_{n,L}\lambda)^{\frac{1}{2}},\quad d_{2n}=\frac{2C_{L}\left((c_{n,L}\lambda)^{-\frac{1}{2}}\right)}{c},\quad
d_{3n}=\frac{C_{L}\left((c_{n,L}\lambda)^{-\frac{1}{2}}\right)(c_{n,L}\lambda)^{-\frac{1}{2}}+1}{c}.\]
\end{theorem}

\begin{proof}
By the definition of $f_{D,\lambda}^{W}$,
\begin{equation}\label{ineq:mini_fw}
  \lambda\left\|f_{D,\lambda}^{W}\right\|^{2}_{\mathcal{H}}+R_{L_{\widehat{W}},D}\left(f_{D,\lambda}^{W}\right)\leq \lambda\left\|f_{\prob,\lambda}\right\|^{2}_{\mathcal{H}}+R_{L_{\widehat{W}},D}\left(f_{\prob,\lambda}\right).
\end{equation}
Recall that
\[A_2(\lambda)=\lambda\|f_{\prob,\lambda}\|^{2}_{\mathcal{H}}+R_{L,\prob}(f_{\prob,\lambda})-
\inf_{f\in\mathcal{H}}R_{L,\prob}(f).\]
Let
\[A^{W}(\lambda)=\lambda\left\|f_{D,\lambda}^{W}\right\|^{2}_{\mathcal{H}}+R_{L,\prob}\left(f_{D,\lambda}^{W}\right)
   -\inf_{f\in \mathcal{H}}R_{L,\prob}(f).\]
Hence,
\begin{align*}
   &A^{W}(\lambda)-A_2(\lambda) \\
   &=\lambda\left\|f_{D,\lambda}^{W}\right\|^{2}_{\mathcal{H}}+R_{L,\prob}\left(f_{D,\lambda}^{W}\right)-
   \lambda\|f_{\prob,\lambda}\|^{2}_{\mathcal{H}}-
   R_{L,\prob}(f_{\prob,\lambda})+R_{L_{\widehat{W}},D}\left(f_{D,\lambda}^{W}\right)-
   R_{L_{\widehat{W}},D}\left(f_{D,\lambda}^{W}\right)\\
   &\leq \lambda\|f_{\prob,\lambda}\|^{2}_{\mathcal{H}}+R_{L_{\widehat{W}},D}(f_{\prob,\lambda})
   +R_{L,\prob}\left(f_{D,\lambda}^{W}\right)
   -R_{L,\prob}(f_{\prob,\lambda})-R_{L_{\widehat{W}},D}\left(f_{D,\lambda}^{W}\right)
   -\lambda\|f_{\prob,\lambda}\|^{2}_{\mathcal{H}}\\
   &=R_{L_{\widehat{W}},D}(f_{\prob,\lambda})-R_{L,\prob}(f_{\prob,\lambda})
   +R_{L,\prob}\left(f_{D,\lambda}^{W}\right)-R_{L_{\widehat{W}},D}\left(f_{D,\lambda}^{W}\right)
   \equiv B^{W}(\lambda),
\end{align*}
where the inequality follows from \eqref{ineq:mini_fw}.

Note that by \eqref{eq:IPW_EQ},
\[R_{L,\prob}(f)=\E\left\{L\left(Y,f(X)\right)\right\}=
\E\left[\E\left\{\frac{ML\left(Y,f(X)\right)}{\pi(X)}\bigg|X,Y\right\}\right]
=R_{L_{W},\prob}(f),\]
where the second equality holds by conditional expectation and the third equality holds for the MAR missing mechanism. Hence,

\[B^{W}(\lambda)=R_{L_{\widehat{W}},D}(f_{\prob,\lambda})-R_{L_{W},\prob}(f_{\prob,\lambda})+
R_{L_{W},\prob}\left(f_{D,\lambda}^{W}\right)-R_{L_{\widehat{W}},D}\left(f_{D,\lambda}^{W}\right).\]

Therefore,
\begin{align*}
 B^{W}(\lambda)&=R_{L_{\widehat{W}},D}(f_{\prob,\lambda})-R_{L_{W},D}(f_{\prob,\lambda})
  +R_{L_{W},D}(f_{\prob,\lambda})-R_{L_{W},\prob}(f_{\prob,\lambda})\\
  &\quad+R_{L_{W},\prob}\left(f_{D,\lambda}^{W}\right)-R_{L_{W},D}\left(f_{D,\lambda}^{W}\right)
  +R_{L_{W},D}\left(f_{D,\lambda}^{W}\right)-R_{L_{\widehat{W}},D}\left(f_{D,\lambda}^{W}\right)\\
  &\leq \left|R_{L_{W},\prob}\left(f_{D,\lambda}^{W}\right)
  -R_{L_{W},D}\left(f_{D,\lambda}^{W}\right)\right|
  +\left|R_{L_{W},D}(f_{\prob,\lambda})-R_{L_{W},\prob}(f_{\prob,\lambda})\right|\\
  &\quad+\left|R_{L_{\widehat{W}},D}(f_{\prob,\lambda})-R_{L_{W},D}\left(f_{\prob,\lambda}\right)\right|
  +\left|R_{L_{W},D}\left(f_{D,\lambda}^{W}\right)-R_{L_{\widehat{W}},D}\left(f_{D,\lambda}^{W}\right)\right|\\
 &\equiv A_{n}+B_{n}+C_{n}+D_{n}.
\end{align*}

We first bound expressions $A_{n}$ and $B_{n}$.
Note that $L(y,0)\leq1$ for all $y\in\mathcal{Y}$.
By Assumption $\ref{assp:bond_PHat}$, $L_{\widehat{W}}\left(\widehat{\pi},M,X,Y,0\right)=
\frac{ML(Y,0)}{\widehat{\pi}(X)}\leq\frac{1}{c_{n,L}}$. Thus,
\[\lambda\left\|f_{D,\lambda}^{W}\right\|_{\mathcal{H}}^{2}\leq R_{L_{\widehat{W}},D}(f_{0})\leq\frac{1}{c_{n,L}},\]
for $f_{0}(X)\equiv0$ for all $X$.

By Assumption $\ref{MAR}$, for every $f\in (c_{n,L}\lambda)^{-\frac{1}{2}}B_{\mathcal{H}}$,
where $B_{\mathcal{H}}$ is the unit ball of $\mathcal{H}$.
\begin{align}\notag
 L_{W}\left(\pi,M,X,Y,f(X)\right)
   &\leq \left|L_{W}\left(\pi,M,X,Y,f(X)\right)-
   L_{W}\left(\pi,M,X,Y,0\right)\right|+\left|L_{W}\left(\pi,M,X,Y,0\right)\right|\\ \notag
   & \leq \left|\frac{M\left\{L(Y,f(X))-L(Y,0)\right\}}{\pi(X)}\right|+\frac{1}{2c}\\ \notag
   & \leq \left|\frac{L(Y,f(X))-L(Y,0)}{\pi(X)}\right|+\frac{1}{2c}\\\notag
   & \leq \frac{1}{2c}\left\{L(Y,f(X))-L(Y,0)\right\}+\frac{1}{2c}\\ \label{bound:B}
   & \leq \frac{1}{2c}\left\{C_{L}\left((c_{n,L}\lambda)^{-\frac{1}{2}}\right)
   (c_{n,L}\lambda)^{-\frac{1}{2}}+1\right\}:\equiv Q_n,
\end{align}
where $c$ is defined in Assumption $\ref{MAR}$ and $C_L(\cdot)$ is a Lipschiz constant defined in Section $\ref{sec:preli}$.

Let $\mathcal{F}_{\varepsilon}$ be an $\varepsilon$-net with cardinality $\left|\mathcal{F}_{\varepsilon}\right|=\mathcal{N}\left((c_{n,L}\lambda)^{-\frac{1}{2}}B_{\mathcal{H}},\|\cdot\|_{\infty},\varepsilon\right)
=\mathcal{N}\left(B_{\mathcal{H}},\|\cdot\|_{\infty},(c_{n,L}\lambda)^{\frac{1}{2}}\varepsilon\right)$. For every function $f\in (c_{n,L}\lambda)^{-\frac{1}{2}}B_{\mathcal{H}}$,
there exists a function $g\in \mathcal{F}_{\varepsilon}$ such that $\|f-g\|_{\infty} \leq \varepsilon$. Since,
\[\left|R_{L_{W},\prob}(f)-R_{L_{W},\prob}(g)\right|
=\left|\E\left\{\frac{L( Y,f(X))-L( Y,g(X))}{\pi(X)}\right\}\right|
  \leq \frac{1}{2c}C_{L}\left((c_{n,L}\lambda)^{-\frac{1}{2}}\right)\varepsilon.\]
This inequality also holds for $\left|R_{L_{W},D}(f)-R_{L_{W},D}(g)\right|$. Thus,
\begin{align*}
  &\left|R_{L_{W},\prob}(f)- R_{L_{W},D}(f)\right|\\
  & \leq \left|R_{L_{W},\prob}(f)-R_{L_{W},\prob}(g)\right|
  +\left|R_{L_{W},D}(f)-R_{L_{W},D}(g)\right|+\left|R_{L_{W},\prob}(g)-R_{L_{W},D}(g)\right|\\
  & \leq \frac{1}{c}C_{L}\left((c_{n,L}\lambda)^{-\frac{1}{2}}\right)\varepsilon
  +\left|R_{L_{W},\prob}(g)-R_{L_{W},D}(g)\right|~\textrm{for}~\textrm{some}~
  g \in \mathcal{F}_{\varepsilon}.
\end{align*}

Using Hoeffding's inequality \citep[Theorem 6.10]{Steinwart:Christmann:08}, and
similarly to the proof of Theorem~6.25 therein,
for any $\eta>0$, we have
\begin{align*}
  \prob\left(A_n+B_n\geq Q_n\sqrt{\frac{2\eta}{n}}
  +\frac{2}{c}C_{L}\left((c_{n,L}\lambda)^{-\frac{1}{2}}\right)\varepsilon\right)&
  \leq\prob\left(2\sup_{g\in\mathcal{F}_{\varepsilon}}\left|R_{L_{W},\prob}(g)-R_{L_{W},D}(g)\right|
   \geq Q_n\sqrt{\frac{2\eta}{n}}\right)\\
   & \leq\sum_{g\in\mathcal{F}_{\varepsilon}}\prob\left(\left|R_{L_{W},\prob}(g)-R_{L_{W},D}(g)\right|
   \geq Q_n\sqrt{\frac{\eta}{2n}}\right)\\
&\leq2\mathcal{N}\left(B_{\mathcal{H}},\|\cdot\|_{\infty},(c_{n,L}\lambda)^{\frac{1}{2}}\varepsilon\right)
e^{-\eta}.
\end{align*}

Elementary algebraic transformation shows that for fixed $\lambda>0$, $n\geq1$, $\varepsilon>0$, and $\eta>0$, with probability not less than $1-e^{-\eta}$,

\begin{align}\label{Bound:AnBn_WCC}
 &A_n+B_n\\\notag
 &\leq\frac{2C_{L}\left((c_{n,L}\lambda)^{-\frac{1}{2}}\right)\varepsilon}{c}+\frac{C_{L}\left((c_{n,L}\lambda)^{-\frac{1}{2}}\right)
 (c_{n,L}\lambda)^{-\frac{1}{2}}+1}{2c}
\left[\sqrt{\frac{2\eta+2\log\left\{2\mathcal{N}
\left(B_{\mathcal{H}},\|\cdot\|_{\infty},(c_{n,L}\lambda)^{\frac{1}{2}}\varepsilon\right)\right\}}
{n}}\right]
\end{align}

Next, we bound $C_n+D_n$.
\begin{align*}
  \left|R_{L_{W},D}(f)-R_{L_{\widehat{W}},D}(f)\right| &=\left|\mathbb{P}_{n}\left\{\frac{ML(Y,f(X))}{\pi(X)}-\frac{ML(Y,f(X))}{\widehat{\pi}(X)}\right\}\right| \\
   &=\left|\mathbb{P}_{n}\left[\frac{ML(Y,f(X))}{\pi(X)\widehat{\pi}(X)}
   \left\{\widehat{\pi}(X)-\pi(X)\right\}\right]\right|\\
   &\leq \frac{C_{L}\left((c_{n,L}\lambda)^{-\frac{1}{2}}\right)(c_{n,L}\lambda)^{-\frac{1}{2}}+1}{2c\cdot c_{n,L}}Err_{1,n}.
\end{align*}
Then
\begin{equation}\label{Bound:CnDnWCC}
C_n+D_n\leq \frac{C_{L}\left((c_{n,L}\lambda)^{-\frac{1}{2}}\right)(c_{n,L}\lambda)^{-\frac{1}{2}}+1}{c\cdot c_{n,L}}Err_{1,n}
\end{equation}

By the definition of $A^{W}(\lambda)$ and $B^{W}(\lambda)$,
\[A^{W}(\lambda)-A_2(\lambda)\leq B^{W}(\lambda)\leq A_n+B_n+C_n+D_n,\]
and the result thus follows from $\eqref{Bound:AnBn_WCC}$ and $\eqref{Bound:CnDnWCC}$.
\end{proof}

\subsection{Proof of Theorem $\ref{theo:Cons_WCC}$}
\begin{proof}
By Condition $\ref{condi:LipCons}$,
\[C_{L}\left((c_{n,L}\lambda)^{-\frac{1}{2}}\right)\leq r(c_{n,L})^{-\frac{q}{2}}\lambda^{-\frac{q}{2}}.\]
Then, together with Condition $\ref{condi:entroy}$, we have
\[\sqrt{\frac{2\log\left\{2\mathcal{N}\left(B_{\mathcal{H}},\|\cdot\|_{\infty},(c_{n,L}\lambda)^{\frac{1}{2}}\varepsilon\right)\right\}}{n}}
\leq\sqrt{\frac{2\ln2+2a\left\{(c_{n,L}\lambda)^{\frac{1}{2}}\varepsilon\right\}^{-2p}}{n}}
\leq \sqrt{\frac{2a}{n}}\left\{(c_{n,L}\lambda)^{\frac{1}{2}}\varepsilon\right\}^{-p}.\]
Therefore,
\begin{align*}
&A^{W}(\lambda)-A_2(\lambda)\\
&\leq\frac{r(c_{n,L})^{-\frac{q}{2}}\lambda^{-\frac{q}{2}}(c_{n,L}\lambda)^{-\frac{1}{2}}+1}{c}\left[2(c_{n,L}\lambda)^{\frac{1}{2}}\varepsilon
+\sqrt{\frac{2a}{n}}\left\{(c_{n,L}\lambda)^{\frac{1}{2}}\varepsilon\right\}^{-p}+\left(\frac{2\eta}{n}\right)^{\frac{1}{2}}
+\frac{Err_{1,n}}{c_{n,L}}\right].
\end{align*}
Let
\[\varepsilon=(c_{n,L}\lambda)^{-\frac{1}{2}}\left(\frac{p}{2}\right)^{\frac{1}{p+1}}
\left(\frac{2a}{n}\right)^{\frac{1}{2p+2}}.\]
Then using some algebra,
\[2(c_{n,L}\lambda)^{\frac{1}{2}}\varepsilon+\sqrt{\frac{2a}{n}}\left\{(c_{n,L}\lambda)^{\frac{1}{2}}\varepsilon\right\}^{-p}
=(p+1)\left(\frac{2}{p}\right)^{\frac{p}{p+1}}\left(\frac{2a}{n}\right)^{\frac{1}{2p+2}}\leq 3\left(\frac{2a}{n}\right)^{\frac{1}{2p+2}},\]
where the last inequality can be verified by \cite[Lemma A.1.5]{Steinwart:Christmann:08}.

Since $\left|\widehat{\pi}(X)-\widehat{\pi}(X)\right|=\O_p\left(n^{-\frac{1}{2}}\right)$, we have
$Err_{1,n}=\O_p\left(n^{-\frac{1}{2}}\right)$, there exists a constant $b_{1}(\eta)$ such that
for all $n\geq1$
\[\prob\left(Err_{1,n}\geq b_{1}(\eta)n^{-\frac{1}{2}}\right)<e^{-\eta}.\]

For fixed $\lambda>0$, $n\geq1$, $\varepsilon>0$, and $\eta>0$, with probability not less than $1-2e^{-\eta}$,
\begin{equation}\label{ineq:univConsiWCC}
A^{W}(\lambda)- A_2(\lambda)
 \leq\frac{r(c_{n,L})^{-\frac{q+1}{2}}\lambda^{-\frac{q+1}{2}}+1}{c}\left[
 3\left(\frac{2a}{n}\right)^{\frac{1}{2p+2}}+\left(\frac{2\eta}{n}\right)^{\frac{1}{2}}
 +\frac{b_{1}(\eta)n^{-\frac{1}{2}}}{c_{n,L}}\right].
\end{equation}

Note that by Assumption $\ref{assp:bond_PHat}$, $c_{n,L}=\O\left(n^{-d}\right)$,
\begin{equation}\label{eq:Lip_WCC}
r(c_{n,L})^{-\frac{q+1}{2}}=\O\left(n^{\frac{(q+1)d}{2}}\right).
\end{equation}

When $\lambda^{\frac{q+1}{2}}n^{\min\left(\frac{1}{2}-d,\frac{1}{2p+2}\right)-\frac{(q+1)d}{2}}\longrightarrow\infty$, the $\mathcal{P}$-universal consistency holds.
\end{proof}

\subsection{Proof of Corollary $\ref{theo:Cons_WCC}$}
\begin{proof}

By \eqref{ineq:univConsiWCC} and Assumption~\ref{assp:BoundA2},
\begin{align*}
A^{W}(\lambda) &\leq A_2(\lambda)+
 \leq\frac{r(c_{n,L})^{-\frac{q+1}{2}}\lambda^{-\frac{q+1}{2}}+1}{c}\left[
 3\left(\frac{2a}{n}\right)^{\frac{1}{2p+2}}+\left(\frac{2\eta}{n}\right)^{\frac{1}{2}}
 +\frac{b_{1}(\eta)n^{-\frac{1}{2}}}{c_{n,L}}\right]\\
&\leq b\lambda^{\gamma}+\left(\frac{r(c_{n,L})^{-\frac{q+1}{2}}}{c}\lambda^{-\frac{q+1}{2}}+\frac{1}{c}\right)
\left[3\left(\frac{2a}{n}\right)^{\frac{1}{2p+2}}+\left(\frac{2\eta}{n}\right)^{\frac{1}{2}}
 +\frac{b_{1}(\eta)n^{-\frac{1}{2}}}{c_{n,L}}\right].
\end{align*}
Let
\[G_1(\lambda)=b\lambda^{\gamma}+\left(\frac{r(c_{n,L})^{-\frac{q+1}{2}}}{c}\lambda^{-\frac{q+1}{2}}+\frac{1}{c}\right)
\left[3\left(\frac{2a}{n}\right)^{\frac{1}{2p+2}}+\left(\frac{2\eta}{n}\right)^{\frac{1}{2}}
 +\frac{b_{1}(\eta)n^{-\frac{1}{2}}}{c_{n,L}}\right].\]
Taking the derivative with respect to $\lambda$ and setting it equal to $0$,
\[b\gamma\lambda^{\gamma-1}=\frac{r(c_{n,L})^{-\frac{q+1}{2}}}{c}\frac{q+1}{2}\lambda^{-\frac{q+3}{2}}
\left[3\left(\frac{2a}{n}\right)^{\frac{1}{2p+2}}+\left(\frac{2\eta}{n}\right)^{\frac{1}{2}}
+\frac{b_{1}(\eta)n^{-\frac{1}{2}}}{c_{n,L}}\right].\]
By \eqref{eq:Lip_WCC} and Assumption \ref{assp:bond_PHat}
\begin{align*}
  \lambda^{\gamma+\frac{q+1}{2}}
  &\propto\left(\frac{1}{n}\right)^{\min\left(\frac{1}{2p+2},\frac{1}{2}-d\right)}n^{\frac{(q+1)d}{2}},\\
\Rightarrow\lambda &\propto n^{\left\{-\min\left(\frac{1}{2p+2},\frac{1}{2}-d\right)+\frac{(q+1)d}{2}\right\}\frac{2}{2\gamma+q+1}}.
\end{align*}
Note that by choosing large $r$ where $r$ is defined in Condition $\ref{condi:LipCons}$, $G_1''\left(n^{\left\{-\min\left(\frac{1}{2p+2},\frac{1}{2}-d\right)
+\frac{(q+1)d}{2}\right\}\frac{2}{2\gamma+q+1}}\right)$
can be positive. Then, for $\lambda=n^{\left\{-\min\left(\frac{1}{2p+2},\frac{1}{2}-d\right)+\frac{(q+1)d}{2}\right\}\frac{2}{2\gamma+q+1}}$,
\begin{align*}
&G_1(\lambda) \\ &=bn^{\left\{-\min\left(\frac{1}{2p+2},\frac{1}{2}-d\right)+\frac{(q+1)d}{2}\right\}\frac{2\gamma}{2\gamma+q+1}}\\
 &+\left\{\frac{r(c_{n,L})^{-\frac{q+1}{2}}}{c}n^{\left\{-\min\left(\frac{1}{2p+2},\frac{1}{2}-d\right)
 +\frac{(q+1)d}{2}\right\}
 \frac{2}{2\gamma+q+1}\left(-\frac{q+1}{2}\right)}+\frac{1}{c}\right\}
 \left[3\left(\frac{2a}{n}\right)^{\frac{1}{2p+2}}+\left(\frac{2\eta}{n}\right)^{\frac{1}{2}}
+\frac{b_{1}(\eta)n^{-\frac{1}{2}}}{c_{n,L}}\right]\\
&\leq bn^{\left\{-\min\left(\frac{1}{2p+2},\frac{1}{2}-d\right)+\frac{(q+1)d}{2}\right\}\frac{2\gamma}{2\gamma+q+1}}\\
&\quad+c_{\text{\prob}}\left(c_{a}+\sqrt{\eta}+b_1(\eta)\right)
n^{\left\{-\min\left(\frac{1}{2p+2},\frac{1}{2}-d\right)+\frac{(q+1)d}{2}\right\}
 \frac{2}{2\gamma+q+1}\left(-\frac{q+1}{2}\right)-\min\left(\frac{1}{2p+2},\frac{1}{2}-d\right)+\frac{(q+1)d}{2}}\\
&\leq Q\left(\sqrt{\eta}+b_1(\eta)+c_{a,b}\right)
n^{\left\{-\min\left(\frac{1}{2p+2},\frac{1}{2}-d\right)+\frac{(q+1)d}{2}\right\}
 \frac{2\gamma}{2\gamma+q+1}},
\end{align*}
where $c_{a}$ is a constant related to $a$,
$c_{a,b}$ is a constant related to $a$ and $b$,
$c_{\text{\prob}}$ and $Q$ are constants related to $\prob$.
None of them is related to $\eta$.

By Assumption $\ref{assp:BoundA2}$, for fixed $\lambda>0$, $n\geq1$, $\varepsilon>0$, and $\eta>0$, with probability not less than $1-2e^{-\eta}$,
\[A^{W}(\lambda)\leq Q\left(\sqrt{\eta}+b_1(\eta)+c_{a,d}\right)
   n^{\left\{-\min\left(\frac{1}{2p+2},\frac{1}{2}-d\right)+\frac{(q+1)d}{2}\right\}
   \frac{2\gamma}{2\gamma+q+1}}.\]

Therefore, the learning rate is $n^{\left\{-\min\left(\frac{1}{2p+2},\frac{1}{2}-d\right)+\frac{(q+1)d}{2}\right\}\frac{2\gamma}{2\gamma+q+1}}$.
\end{proof}

\subsection{Oracle Inequality for Doubly-Robust Kernel Machines}
\begin{theorem}\label{ineq:OIDR}
Let Assumptions~\ref{MAR} and \ref{assp:bond_PHat} hold.
When $L$ is a quadratic loss, for fixed $\lambda > 0$, $n \geq 1$, $\varepsilon > 0$, and $\eta > 0$,
 with probability not less than $1-e^{-\eta}$
\begin{align*}
&\lambda\left\|f_{D,\lambda}^{DR}\right\|^{2}_{\mathcal{H}}+R_{L,\prob}\left(f_{D,\lambda}^{DR}\right)-\inf_{f\in \mathcal{H}}R_{L,\prob}(f)\\
 &<A_2(\lambda)+u_{2n}(\lambda)\varepsilon
 +3u_{3n}(\lambda)\left[\sqrt{\frac{2\eta+
 2\log\left\{2\mathcal{N}(B_{\mathcal{H}},\|\cdot\|_{\infty},u_{1n}(\lambda)\varepsilon)\right\}}{n}}\right]
 +\frac{2u_{3n}(\lambda)}{c_{n,L}}Err_{1,n}\\
 &\quad+2\left(\frac{1}{c_{n,L}}+1\right)Err_{2,n},
\end{align*}
where
\[u_{1n}=(c_{2,n}\lambda)^{\frac{1}{2}},\quad u_{2n}=\frac{6r(c_{2,n}\lambda)^{-\frac{1}{2}}}{c},
\quad u_{3n}=\frac{r(c_{2,n}\lambda)^{-1}+1}{c},\quad c_{2,n}=\frac{c_{n,L}}{2+c_{n,U}},\]
$c_{n,L}$  and $c_{n,U}$ are defined as in Assumption $\ref{assp:bond_PHat}$.
\end{theorem}
\begin{proof}
Let $c_1=\frac{3}{2c}$, where $c$ is defined as in Assumption $\ref{MAR}$.
Since $L(Y,0)\leq 1$, we also have $H(X,\beta_0,0)=\E\left\{L(Y,0)\mid X\right\}\leq 1$.
Recall that
\[L_{W,H}\left(\pi,H,M,X,Y,f(X)\right)\equiv\frac{ML(Y,f(X))}{\pi(X)}-
\frac{M-\pi(X)}{\pi(X)}H(X,\beta_0,f(X)).\]
By Assumption \ref{MAR},
 \begin{align*}
\left|L_{W,H}\left(\pi,H,M,X,Y,0\right)\right|&=\left|\frac{ML(Y,0)}{\pi(X)}-
\frac{M-\pi(X)}{\pi(X)}H(X,\beta_0,0)\right| \\
 & \leq \frac{ML(Y,0)}{\pi(X)}+
\frac{M+\pi(X)}{\pi(X)}H(X,\beta_0,0)\\
& \leq \frac{2M+\pi(X)}{\pi(X)}\leq\frac{3}{2c}\equiv c_1,
\end{align*}
where the last inequality follows from Assumption~\ref{MAR}.

Since $\widehat{H}(X,0)=\int_{y\in\mathcal{Y}} L(y,0)dF_{Y|X}(y\mid X,\widehat{\beta})\leq 1$,
 \begin{align*}
\left|L_{\widehat{W},\widehat{H}}\left(\widehat{\pi},\widehat{H},M,X,Y,0\right)\right|
&=\left|\frac{ML(Y,0)}{\widehat{\pi}(X)}-
\frac{M-\widehat{\pi}(X)}{\widehat{\pi}(X)}\widehat{H}(X,0)\right| \\
 & \leq \frac{ML(Y,0)}{\widehat{\pi}(X)}+
\frac{M+\widehat{\pi}(X)}{\widehat{\pi}(X)}\widehat{H}(X,0)\\
& \leq \frac{2M+\widehat{\pi}(X)}{\widehat{\pi}(X)}\leq \frac{2+c_{n,U}}{c_{n,L}}\equiv\frac{1}{c_{2,n}}.
\end{align*}

Note that
\begin{equation}\label{ineq:bound_f_DR}
\lambda\left\|f_{D,\lambda}^{DR}\right\|^{2}_{\mathcal{H}}\leq
\lambda\left\|f_{D,\lambda}^{DR}\right\|^{2}_{\mathcal{H}}+
R_{L_{\widehat{W},\widehat{H}},D}\left(f_{D,\lambda}^{DR}\right)
\leq R_{L_{\widehat{W},\widehat{H}},D}(0)\leq \frac{1}{c_{2,n}}.
\end{equation}
For every $f\in (c_{2,n}\lambda)^{-\frac{1}{2}}B_{\mathcal{H}}$,
\begin{equation}\label{Bound:LDR}
\left|L(Y,f(X))\right|\leq\left|L(Y,f(X))-L(Y,0)\right|+L(Y,0)
\leq C_{L}\left((c_{2,n}\lambda)^{-\frac{1}{2}}\right)(c_{2,n}\lambda)^{-\frac{1}{2}}+1
\leq r(c_{2,n}\lambda)^{-1}+1
\end{equation}
The last inequality holds since by Condition~\ref{condi:LipCons}, for the quadratic
loss $C_{L}\left((c_{2,n}\lambda)^{-\frac{1}{2}}\right)\leq r (c_{2,n}\lambda)^{-\frac{1}{2}}$.

Using the
argument as in $\eqref{bound:B}$ in Theorem~\ref{ineq:OIWCC},
\[L_{W,H}\left(\pi,H,M,X,Y,f(X)\right)\leq
\frac{3\left\{r(c_{2,n}\lambda)^{-1}+1\right\}}{2c}
\equiv T_n.\]
Let
\[A^{DR}(\lambda) =\lambda\left\|f_{D,\lambda}^{DR}\right\|^{2}_{\mathcal{H}}
+R_{L,\prob}\left(f_{D,\lambda}^{DR}\right)-\inf_{f\in H}R_{L,\prob}(f).\]
Since $R_{L,\prob}(f)=R_{L_{W,H},\prob}(f)$, using the same technique as in the proof of Theorem~\ref{ineq:OIWCC},
\begin{align}\notag
A^{DR}(\lambda)-A_2(\lambda)
   &\leq \left|R_{L_{W,H},\prob}(f_{\prob,\lambda})-R_{L_{W,H},D}(f_{\prob,\lambda})\right|
  +\left|R_{L_{W,H},\prob}\left(f_{D,\lambda}^{DR}\right)-R_{L_{W,H},D}\left(f_{D,\lambda}^{DR}\right)\right|\\ \notag
  &\quad+\left|R_{L_{W,H},D}(f_{\prob,\lambda})-R_{L_{\widehat{W},\widehat{H}},D}(f_{\prob,\lambda})\right|
  +\left|R_{L_{W,H},D}\left(f_{D,\lambda}^{DR}\right)-R_{L_{\widehat{W},\widehat{H}},D}\left(f_{D,\lambda}^{DR}\right)\right|\\ \label{Bound}
  &\equiv A_n+B_n+C_n+D_n
\end{align}

Let $\mathcal{F}_{\varepsilon}$ be an $\varepsilon$-net of $(c_{2,n}\lambda)^{-\frac{1}{2}}B_{\mathcal{H}}$ with cardinality $\left|\mathcal{F}_{\varepsilon}\right|=\mathcal{N}\left((c_{2,n}\lambda)^{-\frac{1}{2}}B_{\mathcal{H}}\|\cdot\|_{\infty},\varepsilon\right)
=\allowbreak\mathcal{N}(B_{\mathcal{H}},\|\cdot\|_{\infty},(c_{2,n}\lambda)^{\frac{1}{2}}\varepsilon)$. For every function $f\in (c_{2,n}\lambda)^{-\frac{1}{2}}B_{\mathcal{H}}$,
there exists a function $g\in \mathcal{F}_{\varepsilon}$, such that $\|f-g\|_{\infty} \leq \varepsilon$.
Thus,
\[\left|L(Y,f(X))-L(Y,g(X))\right| \leq C_{L}\left((c_{2,n}\lambda)^{-\frac{1}{2}}\right)\|f-g\|_{\infty} \leq r(c_{2,n}\lambda)^{-\frac{1}{2}}\varepsilon,\]
and
\begin{align*}
 \left|H(X,\beta_0,f(X))-H(X,\beta_0,g(X))\right|& = \left|\int_{y\in\mathcal{Y}}\{L(y,f(X))-L(y,g(X))\}dF_{Y|X}(y\mid X,\beta_0)\right|\\
   & \leq \int_{y\in\mathcal{Y}}\left|L(y,f(X))-L(y,g(X))\right|dF_{Y|X}(y\mid X,\beta_0)\\
   & \leq r(c_{2,n}\lambda)^{-\frac{1}{2}}\varepsilon.
\end{align*}

Therefore
\begin{align}\label{dif_fg}
 &\left|R_{L_{W,H},\prob}(f)-R_{L_{W,H},\prob}(g)\right| \\ \notag
 &=\left|\E\left[\frac{M\{L(Y,f(X))-L(Y,g(X))\}}{\pi(X)}-
 \frac{\{M-\pi(X)\}\{H(X,\beta_0,f(X))-H(X,\beta_0,g(X))\}}{\pi(X)}\right]\right| \\ \notag
  &\leq \E\left[\left|\frac{M\{L(Y,f(X))-L(Y,g(X))\}}{\pi(X)}-
 \frac{\{M-\pi(X)\}\{H(X,\beta_0,f(X))-H(X,\beta_0,g(X))\}}{\pi(X)}\right|\right]\\ \notag
 &\leq \E\left[\frac{M\left|L(Y,f(X))-L(Y,g(X))\right|}{\pi(X)}+
 \frac{\{M+\pi(X)\}\left|H(X,\beta_0,f(X))-H(X,\beta_0,g(X))\right|}{\pi(X)}\right]\\ \notag
 &\leq \frac{3r(c_{2,n}\lambda)^{-\frac{1}{2}}\varepsilon}{2c}.
\end{align}
Similarly,
\[\left|R_{L_{W,H},D}(f)-R_{L_{W,H},D}(g)\right|
  \leq \frac{3r(c_{2,n}\lambda)^{-\frac{1}{2}}\varepsilon}{2c}.\]
Using $\eqref{dif_fg}$ we can bound $A_n$ and $B_n$ of $\eqref{Bound}$
\begin{align*}
  &\left|R_{L_{W,H},\prob}(f)-R_{L_{W,H},D}(f)\right| \\
  & \leq \left|R_{L_{W,H},\prob}(f)-R_{L_{W,H},\prob}(g)\right|+\left|R_{L_{W,H},D}(f)-R_{L_{W,H},D}(g)\right|
   +\left|R_{L_{W,H},\prob}(g)-R_{L_{W,H},D}(g)\right| \\
  & \leq \frac{3r(c_{2,n}\lambda)^{-\frac{1}{2}}\varepsilon}{c}
  +\left|R_{L_{W,H},\prob}(g)-R_{L_{W,H},D}(g)\right|.
\end{align*}
Using the similar argument as \cite[Theorem 6.25]{Steinwart:Christmann:08} for any $\eta>0$, we have
\begin{align*}
   &\prob\left(A_n+B_n\geq
   T_n\sqrt{\frac{2\eta}{n}}+\frac{6r(c_{2,n}\lambda)^{-\frac{1}{2}}\varepsilon}{c}\right)\\
    & \leq\prob\left(2\sup_{g\in\mathcal{F}_{\varepsilon}}
     \left|R_{L_{W,H},\prob}(g)-R_{L_{W,H},D}(g)\right|
    \geq T_n\sqrt{\frac{2\eta}{n}}\right)\\
   & \leq\sum_{g\in\mathcal{F}_{\varepsilon}}
   \prob\left(\left|R_{L_{W,H},\prob}(g)-R_{L_{W,H},D}(g)\right|
   \geq T_n\sqrt{\frac{\eta}{2n}}\right)\\
   &\leq2\mathcal{N}\left(B_{\mathcal{H}},\|\cdot\|_{\infty},(c_{2,n}\lambda)^{\frac{1}{2}}\varepsilon\right)e^{-\eta},
\end{align*}
where the last inequality is from Hoeffding's inequality \cite[Theorem 6.10]{Steinwart:Christmann:08}.

Elementary algebraic transformation shows that for fixed $\lambda>0$, $n\geq1$, $\varepsilon>0$, and $\eta>0$, with probability not less than $1-e^{-\eta}$,

\begin{equation}\label{Bound_DR_AB}
 A_n+B_n
 \leq\frac{6r(c_{2,n}\lambda)^{-\frac{1}{2}}\varepsilon}{c}
 +\frac{3\left\{r(c_{2,n}\lambda)^{-1}+1\right\}}{c}
\left[\sqrt{\frac{2\eta+2\log\left\{
2\mathcal{N}\left(B_{\mathcal{H}},\|\cdot\|_{\infty},(c_{2,n}\lambda)^{\frac{1}{2}}\varepsilon\right)\right\}}{n}}\right].
\end{equation}
Next bound $C_n$ and $D_n$,
\begin{align}\label{ineq:boud_R_DR}
&\left|R_{L_{W,H},D}(f)-R_{L_{\widehat{W},\widehat{H}},D}(f)\right| \\\notag
&=\mathbb{P}_n\left|\frac{ML(Y,f(X))}{\pi(X)}-\frac{ML(Y,f(X))}{\widehat{\pi}(X)}
+\frac{M-\widehat{\pi}(X)}{\widehat{\pi}(X)}\widehat{H}(X,f(X))-
\frac{M-\pi(X)}{\pi(X)}H(X,\beta_0,f(X))\right|\\ \notag
& \leq\mathbb{P}_n\bigg[\frac{ML(Y,f(X))}{\pi(X)\widehat{\pi}(X)}\left|\pi(X)-\widehat{\pi}(X)\right|\\\notag
&\quad\qquad+\left|\widehat{H}(X,f(X))-H(X,\beta_0,f(X))\right|+\left|\frac{M\widehat{H}(X,f(X))}{\widehat{\pi}(X)}
-\frac{MH(X,\beta_0,f(X))}{\pi(X)}\right|\bigg].
\end{align}
Then,
\begin{align}\label{ineq:bound_H}
& \left|\frac{M\widehat{H}(X,f(X))}{\widehat{\pi}(X)}
-\frac{MH(X,\beta_0,f(X))}{\pi(X)}\right|\\\notag
 &=\left|\frac{M\widehat{H}(X,f(X))}{\widehat{\pi}(X)}
 -\frac{MH(X,\beta_0,f(X))}{\widehat{\pi}(X)}
+\frac{MH(X,\beta_0,f(X))}{\widehat{\pi}(X)}-\frac{MH(X,\beta_0,f(X))}{\pi(X)}\right|\\\notag
&\leq \frac{M}{\widehat{\pi}(X)}\left|\widehat{H}(X,f(X))-H(X,\beta_0,f(X))\right|
+\frac{MH(X,\beta_0,f(X))}{\pi(X)\widehat{\pi}(X)}\left|\pi(X)-\widehat{\pi}(X)\right|.
\end{align}
Hence, by inequality \eqref{ineq:boud_R_DR} and \eqref{ineq:bound_H}, and
definition of $Err_{1,n}$ and $Err_{2,n}$ in Subsection~\ref{subsec:ass_con_err},
\begin{align*}
   & \left|R_{L_{W,H},D}(f)-R_{L_{\widehat{W},\widehat{H}},D}(f)\right|\\
   & \leq\frac{L(Y,f(X))}{2c\cdot c_{n,L}}Err_{1,n}+Err_{2,n}+\frac{1}{c_{n,L}}Err_{2,n}
   +\frac{H(X,\beta_0,f(X))}{2c\cdot c_{n,L}}Err_{1,n}.
\end{align*}

Similarly to inequality~\eqref{Bound:LDR},
$\left|H(X,\beta_0,f(X))\right| \leq r(c_{2,n}\lambda)^{-1}+1$.
Then we have
\begin{equation}\label{Bound_DR_CD}
\left|R_{L_{W,H},D}(f)-R_{L_{\widehat{W},\widehat{H}},D}(f)\right|
\leq \frac{r(c_{2,n}\lambda)^{-1}+1}{c\cdot c_{n,L}}Err_{1,n}+
\left(\frac{1}{c_{n,L}}+1\right)Err_{2,n}.
\end{equation}
By  $\eqref{Bound_DR_AB}$ and $\eqref{Bound_DR_CD}$, and using
\[ A^{DR}(\lambda)-A_2(\lambda)\leq A_n+B_n+C_n+D_n,\]
the result follows.
\end{proof}

\subsection{Proof of Lemma $\ref{lem:hn}$}
\begin{proof}
Define $X_i(f)=L(X_i,Y_i,f(X_i))-H(X_i,\beta_0,f(X_i))$ and let
\[\tilde h_n(f)=\frac1n\sum_{i=1}^{n}L(X_i,Y_i,f(X_i))-H(X_i,\beta_0,f(X_i))
=\frac1n\sum_{i=1}^{n}X_i(f).\]
 By inequality \eqref{ineq:bound_f_DR}
we have $\|f\|_{\mathcal{H}}\leq \left(c_{2,n}\lambda\right)^{-\frac{1}{2}}$,
where $c_{2,n}=\frac{c_{n,L}}{2+c_{n,U}}$.
By Cauchy-Schwarz inequality and the reproducing property,
see \cite[Lemma 4.23]{Steinwart:Christmann:08},
\[|f(x)|=\langle f,k(\cdot,x)\rangle\leq\|f(x)\|_{\mathcal{H}}\sqrt{k(x,x)}\leq
\|f(x)\|_{\mathcal{H}}\|k(x,x)\|_{\infty}.\]
Consequently, $\|f\|_{\infty}\leq \|f\|_{\mathcal{H}}\|k(x,x)\|_{\infty}=
\|f\|_{\mathcal{H}}$, where the inequality follows since we assume
that $\|k(x,x)\|\allowbreak\leq 1$.

Since $\|f\|_{\infty}\leq \|f\|_{\mathcal{H}}$,
the space ${\mathcal{H}}_n$ over which the supremum $h_n$ is taken is contained in $\left(c_{2,n}\lambda\right)^{-\frac{1}{2}}B_{\mathcal{H}}$.

By \eqref{Bound:LDR},
$\|X_i(f)\|_{\infty}\leq 2\left\{r\left(c_{2,n}\lambda\right)^{-1}+1\right\}.$
Using the functional Hoeffding's inequality \cite[Section 6.5]{Berestycki:Nickl:Ben:09},
\[\prob\left[\frac{1}{\sqrt{2\left\{r\left(c_{2,n}\lambda\right)^{-1}+1\right\}}}
\left\|\sum_{i=1}^{n}X_i(f)\right\|_{\infty}\geq C
\right]\leq\frac{1}{K_{u}}\exp\left(-\frac{C^{2}}{K_{u}n}\right),\]
where $K_{u}$ is a universal constant and $C$ is any constant.

Let $\widetilde{C}=\frac{C}{\sqrt{n}}$, so $C=\sqrt{n}\widetilde{C}$. Then,
\begin{equation}\label{eq:Hoeffding}
\prob\left[\frac{\sqrt{n}}{\sqrt{2\left\{r\left(c_{2,n}\lambda\right)^{-1}+1\right\}}}
\left\|\frac{1}{n}\sum_{i=1}^{n}X_i(f)\right\|_{\infty}\geq \widetilde{C}
\right]\leq\frac{1}{K_{u}}\exp\left(-\text{const}~\widetilde{C}^{2}\right).
\end{equation}

Since $c_{2,n}=\frac{c_{n,L}}{2+c_{n,U}}$ for $0<c_{n,U}<1$,
and $\frac{1}{c_{n,L}}=\O\left(n^{d}\right)$, then $\frac{1}{c_{2,n}}=\O\left(n^{d}\right)$. Thus,
\begin{equation*}
r\left(c_{2,n}\lambda\right)^{-1}=\O\left(n^{d}\lambda^{-1}\right).
\end{equation*}
Consequently,
\begin{equation}\label{ineq:boud_Lip}
\frac{\sqrt{n}}{\sqrt{2\left\{r\left(c_{2,n}\lambda\right)^{-1}+1\right\}}} =\O\left(n^{\frac{1}{2}-\frac{d}{2}}\lambda^{\frac{1}{2}}\right).
\end{equation}
We have, from~\eqref{eq:Hoeffding} that
$h_n=\O_p\left(n^{-\left(\frac{1}{2}-\frac{d}{2}\right)}\lambda^{-\frac{1}{2}}\right)$.
\end{proof}

\subsection{Proof of Lemma~\ref{Bond:Err2}}
\begin{proof}
Note that for every $f$,
\begin{align*}
 &\left\|\widehat{H}(x,f(x))-H(x,\beta_0,f(x))\right\|_{\infty}\\
 &=\sup_{x\in\mathcal{X}}\left|\int_{y\in\mathcal{Y}}L(y,f(x))dF_{Y|X}\left(y\mid x,\widehat{\beta}\right)-\int_{y\in\mathcal{Y}}L(y,f(x))dF_{Y|X}\left(y\mid x,\beta_0\right)\right|\\
 &=\sup_{x\in\mathcal{X}}\left|\int_{y\in\mathcal{Y}}L(y,f(x))d\left\{F_{Y|X}\left(y\mid x,\widehat{\beta}\right)-F_{Y|X}\left(y\mid x,\beta_0\right)\right\}\right|.
\end{align*}
By \eqref{Bound:LDR}
\[|L(y,f(x))|\leq r(c_{2,n}\lambda)^{-1}+1=
\O\left(n^{d}\lambda^{-1}\right).\]
We have
\begin{align*}
  Err_{2,n}&=\sup_{f\in{\mathcal{H}_n}}\left\|\widehat{H}(x,f(x))-H(x,\beta_0,f(x))\right\|_{\infty}\\
  &\leq\left\{r(c_{2,n}\lambda)^{-1}+1\right\} \sup_{x\in\mathcal{X}}
  \left|\int_{y\in\mathcal{Y}}d\left\{F_{Y|X}\left(y\mid x,\widehat{\beta}\right)-F_{Y|X}\left(y\mid x,\beta_0\right)\right\}\right|.
\end{align*}

Define the function $\phi:\mathcal{B}\mapsto \mathcal L_{\infty}(X)$ by $\phi(\beta)=\int_y dF_{Y|X}\left(y\mid \cdot,\beta\right)$. Note that $\phi$ is Hadamard differentiable as a composite of $\beta\mapsto F_{Y|X}\left(\cdot\mid \cdot,\beta\right)\mapsto \int_y dF_{Y|X}\left(y\mid \cdot,\beta\right)$. The first mapping is Hadamard differentiable by the assumption of continuous differentiability with respect to $\beta$ and the definition of Hadamard differentiability \cite[Section 2.2.4]{Kosorok:08}, and the second by \cite[Lemma 12.3]{Kosorok:08}.
Thus, by the function delta method \cite[Theorem 2.8]{Kosorok:08},
\begin{equation}\label{Order_int}
\int_{y\in\mathcal{Y}}d\left\{F_{Y|X}\left(y\mid x,\widehat{\beta}\right)-F_{Y|X}\left(y\mid x,\beta_0\right)\right\}=\O_p\left(n^{-\frac{1}{2}}\right).
\end{equation}
Consequently, by the definition of convergence in probability \citep[Section 2.2.1]{Kosorok:08}, we conclude that  $Err_{2,n}=\O_p\left(n^{-\frac{1}{2}+d}\lambda^{-1}\right)$.
\end{proof}

\subsection{Proof of Theorem $\ref{theo:Cons_DR}$}
\begin{proof}
In the proof of Theorem~\ref{ineq:OIDR},
\[A^{DR}(\lambda)-A_2(\lambda)\leq A_n+B_n+C_n+D_n,\]
where $A_n$, $B_n$, $C_n$, and $D_n$ are same as defined in $\eqref{Bound}$.
For $A_n+B_n$, we have the same result as $\eqref{Bound_DR_AB}$. Next we
bound $C_n$ and $D_n$ in the two different situations of (\romannumeral1) and (\romannumeral2).

Recall that
\begin{align*}
R_{L_{W,H},D}
&=\mathbb{P}_{n}\left\{L_{W,H}\left(\pi,H,M,X,Y,f(X)\right)\right\}\\
&=\mathbb{P}_{n}\left\{\frac{ML(Y,f(X))}{\pi(X)}-
\frac{M-\pi(X)}{\pi(X)}H(X,\beta_0,f(X))\right\}.
\end{align*}

Case 1: $\left|\hat\pi(X)-\pi(X)\right|=\O_p\left(n^{-\frac{1}{2}}\right)$ which means
$Err_{1,n}=\O_p\left(n^{-\frac{1}{2}}\right)$,
and $\widehat{\beta}\stackrel{\text{P}}\longrightarrow\beta^{\ast}$
where $\beta^{\ast}$ is not necessarily $\beta_0$. Since
\begin{align*}
&R_{L_{\widehat{W},\widehat{H}},D}(f)\\
&=\mathbb{P}_n\left[\frac{ML(Y,f(X))}{\widehat{\pi}(X)}-
\frac{M-\widehat{\pi}(X)}{\widehat{\pi}(X)}\widehat{H}(X,f(X))\right]\\
&=\mathbb{P}_n\left[\frac{ML(Y,f(X))}{\pi(X)}-
\frac{M-\pi(X)}{\pi(X)}\widehat{H}(X,f(X))+\left\{\frac{ML(Y,f(X))-M\widehat{H}(X,f(X))}
{\widehat{\pi}(X)\pi(X)}\right\}
\left(\pi(X)-{\widehat{\pi}(X)}\right)\right].\\
\end{align*}
Then,
\begin{align*}
&\left|R_{L_{\widehat{W},\widehat{H}},D}(f)-R_{L_{W,H},D}(f)\right|\\
&=\Bigg|\mathbb{P}_n\left[\frac{M-\pi(X)}{\pi(X)}
\left\{H(X,\beta_0,f(X))-\widehat{H}(X,f(X))\right\}\right]\\ \notag
&\quad+\mathbb{P}_n\left[
\frac{ML(Y,f(X))-M\widehat{H}(X,f(X))}{\widehat{\pi}(X)\pi(X)}
\left(\pi(X)-{\widehat{\pi}(X)}\right)\right]\Bigg|\\
&\leq\left|\mathbb{P}_n\left[\frac{M-\pi(X)}{\pi(X)}\right]\right|Err_{2,n}+
\frac{r(c_{2,n}\lambda)^{-1}+1}{c\cdot c_{n,L}}Err_{1,n}\\
&=|a_n|Err_{2,n}+\frac{r(c_{2,n}\lambda)^{-1}+1}
{c\cdot c_{n,L}}Err_{1,n}.
\end{align*}
Since both $a_n$ and $Err_{1,n}$ are $\O_p\left(n^{-\frac{1}{2}}\right)$, for every given $\eta>0$,
there exists a constant $b_3(\eta)$ such that for all $n\geq1$,
\[\prob\left(\max\{|a_n|,Err_{1,n}\}>b_{3}(\eta)n^{-\frac{1}{2}}\right)<e^{-\eta}.\]
Note that
\[Err_{2,n}=\sup_{f\in\mathcal{H}_n}\left\|\widehat{H}(x,f(x))-H(x,\beta_0,f(x))\right\|_{\infty}
\leq2\left\{r(c_{2,n}\lambda)^{-1}+1\right\}.\]
Therefore, with probability not less than $1-e^{-\eta}$,
\[C_n+D_n\leq b_{3}(\eta)n^{-\frac{1}{2}}\left[4\left\{r(c_{2,n}\lambda)^{-1}+1\right\}+
\frac{2\left\{r(c_{2,n}\lambda)^{-1}+1\right\}}{c\cdot c_{n,L}}\right].\]
Combining this bound with~\eqref{Bound_DR_AB}, for every fixed $\lambda>0$, $n\geq1$, $\varepsilon>0$, and $\eta>0$, with probability not less
than $1-2e^{-\eta}$,
\begin{align*}
A^{DR}(\lambda)
&\leq A_2(\lambda)+\frac{3\left\{r(c_{2,n}\lambda)^{-1}+1\right\}}{c}
\Bigg[2(c_{2,n}\lambda)^{\frac{1}{2}}\varepsilon\\
&\qquad+\sqrt{\frac{2\eta+2\log\left\{2\mathcal{N}(B_{\mathcal{H}},\|\cdot\|_{\infty},(c_{2,n}\lambda)^{\frac{1}{2}}\varepsilon)\right\}}{n}}
+\frac{4cb_3(\eta)n^{-\frac{1}{2}}}{3}+\frac{b_{3}(\eta)n^{-\frac{1}{2}}}{c_{n,L}}\Bigg].
\end{align*}
Together with Condition~\ref{condi:entroy} and letting
\[\varepsilon=(c_{2,n}\lambda)^{-\frac{1}{2}}
\left(\frac{p}{2}\right)^{\frac{1}{p+1}}\left(\frac{2a}{n}\right)^{\frac{1}{2p+2}},\]
with probability not less
than $1-2e^{\eta}$,
\begin{equation}\label{univConsi2}
A^{DR}(\lambda)\leq A_2(\lambda)+\frac{3\left\{r(c_{2,n}\lambda)^{-1}+1\right\}}{c}\left[
3\left(\frac{2a}{n}\right)^{\frac{1}{2p+2}}+\left(\frac{2\eta}{n}\right)^{\frac{1}{2}}
+\frac{b_{3}(\eta)n^{-\frac{1}{2}}}{c_{n,L}}
+\frac{4cb_3(\eta)n^{-\frac{1}{2}}}{3}\right].
\end{equation}
Since $\frac{1}{c_{2,n}}=\O(n^{d})$, for $\lambda n^{\min\left(\frac{1}{2p+2},\frac{1}{2}-d\right)-d}\longrightarrow\infty$,
the $\mathcal{P}-$universal consistency holds.

Case 2:
 $\left|\widehat{\beta}-\beta_{0}\right|=\O_p\left(n^{-\frac{1}{2}}\right)$, whereas $\widehat{\pi}(X)\stackrel{\text{P}}\longrightarrow
\pi^{\ast}(X)$ which is not necessarily equal to $\pi_0(X)$.
\begin{align*}
R_{L_{\widehat{W},\widehat{H}},D}(f)
&=\mathbb{P}_n\left[\frac{ML(Y,f(X))}{\widehat{\pi}(X)}-
\frac{M-\widehat{\pi}(X)}{\widehat{\pi}(X)}H(X,\beta_0,f(X))\right]\\
&\quad-\mathbb{P}_n\left[\frac{M-\widehat{\pi}(X)}{\widehat{\pi}(X)}
\left\{\widehat{H}(X,f(X))-H(X,\beta_0,f(X))\right\}\right]
\end{align*}
Then,
\begin{align*}
&\left|R_{L_{\widehat{W},\widehat{H}},D}(f)-R_{L_{W,H},D}(f)\right|\\
&=\Bigg|\mathbb{P}_n\left[\frac{ML(Y,f(X))-MH(X,\beta_0,f(X))}
{\widehat{\pi}(X)\pi(X)}
\left\{\pi(X)-\widehat{\pi}(X)\right\}\right]\\
&\quad-\mathbb{P}_n\left[\frac{M-\widehat{\pi}(X)}{\widehat{\pi}(X)}
\left\{\widehat{H}(X,f(X))-H(X,\beta_0,f(X))\right\}\right]\Bigg|\\
&\leq\frac{Err_{1,n}}{2c\cdot c_{n,L}}\left|\mathbb{P}_n\left\{L(Y,f(X))-H(X,\beta_0,f(X))\right\}\right|
+\frac{1+c_{n,U}}{c_{n,L}}Err_{2,n}\\
&\leq h_n\frac{Err_{1,n}}{2c\cdot c_{n,L}}+\frac{1+c_{n,U}}{c_{n,L}}Err_{2,n}.
\end{align*}
Note that $Err_{1,n}\leq 2$. By Lemma \ref{lem:hn}, $h_n=\O_p\left(n^{-\left(\frac{1}{2}-d\right)}\lambda^{-\frac{1}{2}}\right)$,
and thus there exists a constant $b_4(\eta)$ such that for all $n\geq1$,
\[\prob\left\{|h_n|>b_{4}(\eta)n^{-\left(\frac{1}{2}
	-\frac{d}{2}\right)}\lambda^{-\frac{1}{2}}\right\}<e^{-\eta}.\]
By Lemma~\ref{Bond:Err2}, $Err_{2,n}=\O_p\left(n^{-\left(\frac{1}{2}-d\right)}\lambda^{-1}\right)$
\[\prob\left\{Err_{2,n}\geq b_2(\eta)n^{-\left(\frac{1}{2}-d\right)}\lambda^{-1}\right\}< e^{-\eta}.\]

For a fixed $\eta>0$, with probability not less
than $1-2e^{-\eta}$,

\[C_n+D_n \leq
\frac{2(1+c_{n,U})}{c_{n,L}}b_{2}(\eta)n^{-\left(\frac{1}{2}-d\right)}\lambda^{-1}
+\frac{1}{c\cdot c_{n,L}}b_{4}(\eta)n^{-\left(\frac{1}{2}
-\frac{d}{2}\right)}\lambda^{-\frac{1}{2}}.\]
Combining this bound with~\eqref{Bound_DR_AB}, using Condition~\ref{condi:entroy}, and letting
\[\varepsilon=(c_{2,n}\lambda)^{-\frac{1}{2}}
\left(\frac{p}{2}\right)^{\frac{1}{p+1}}\left(\frac{2a}{n}\right)^{\frac{1}{2p+2}},\]
for every fixed $\lambda>0$, $n\geq1$, $\varepsilon>0$, and $\eta>0$, with probability not less
than $1-3e^{-\eta}$,
\begin{align}\label{learnRateDR2}
&A^{DR}(\lambda)-A_2(\lambda)\\\notag
&\leq\frac{3\left\{r(c_{2,n}\lambda)^{-1}+1\right\}}{c}\left\{
3\left(\frac{2a}{n}\right)^{\frac{1}{2p+2}}+\left(\frac{2\eta}{n}\right)^{\frac{1}{2}}\right\}
+\frac{2(1+c_{n,U})}{c_{n,L}}b_{2}(\eta)n^{-\left(\frac{1}{2}-d\right)}\lambda^{-1}\\\notag
&\quad+\frac{1}{c\cdot c_{n,L}}b_{4}(\eta)n^{-\left(\frac{1}{2}-\frac{d}{2}\right)}\lambda^{-\frac{1}{2}}.
\end{align}
Note that $\frac{1+c_{n,U}}{c_{n,L}}=\O\left(n^d\right)$,
and that $1-2d>\frac{1}{2}-d$. Hence,
when $\lambda n^{\min\left(\frac{1}{2p+2},\frac{1}{2}-d\right)-d}\longrightarrow\infty$,
the $\mathcal{P}-$universal consistency holds.
\end{proof}

\subsection{Proof of Corollary~\ref{coro:learnRate_DR}}
\begin{proof}
Case 1: $\left|\hat\pi(X)-\pi(X)\right|=\O_p\left(n^{-\frac{1}{2}}\right)$.
By \eqref{univConsi2} and Assumption~\ref{assp:BoundA2},
\[A^{DR}(\lambda)\leq b\lambda^{\gamma}+\frac{3\left\{r(c_{2,n}\lambda)^{-1}+1\right\}}{c}\left\{
 3\left(\frac{2a}{n}\right)^{\frac{1}{2p+2}}+\left(\frac{2\eta}{n}\right)^{\frac{1}{2}}
 +\frac{b_{3}(\eta)n^{-\frac{1}{2}}}{c_{n,L}}
 +\frac{4cb_3(\eta)n^{-\frac{1}{2}}}{3}\right\}.\]

Let
\[G_2(\lambda)= b\lambda^{\gamma}+\frac{3\left\{r(c_{2,n}\lambda)^{-1}+1\right\}}{c}\left\{
 3\left(\frac{2a}{n}\right)^{\frac{1}{2p+2}}+\left(\frac{2\eta}{n}\right)^{\frac{1}{2}}
 +\frac{b_{3}(\eta)n^{-\frac{1}{2}}}{c_{n,L}} +\frac{4c}{3}b_3(\eta)n^{-\frac{1}{2}}\right\}\]
Taking the derivative with respect to $\lambda$ and setting it equal to $0$,
\[b\gamma\lambda^{\gamma-1}=\frac{3r(c_{2,n})^{-1}}{c}\lambda^{-2}
\left\{3\left(\frac{2a}{n}\right)^{\frac{1}{2p+2}}+\left(\frac{2\eta}{n}\right)^{\frac{1}{2}}
 +\frac{b_{3}(\eta)n^{-\frac{1}{2}}}{c_{n,L}}+\frac{4c}{3}b_1(\eta)n^{-\frac{1}{2}}\right\}.\]
Note that $\frac{1}{c_{2,n}}=\O\left(n^{d}\right)$. Thus,
\begin{align*}
\lambda^{\gamma+1}&
\propto\left(\frac{1}{n}\right)^{\min\left(\frac{1}{2p+2},\frac{1}{2}-d\right)-d} \\
\Rightarrow\lambda &
\propto n^{\left\{-\min\left(\frac{1}{2p+2},\frac{1}{2}-d\right)+d\right\}\frac{1}{\gamma+1}}.
\end{align*}
Note that by choosing large r, we have $G_2''\left(n^{\left\{-\min\left(\frac{1}{2p+2},\frac{1}{2}-d\right)+d\right\}\frac{1}{\gamma+1}}\right)>0$.
Then for
 \[\lambda=n^{\left\{-\min\left(\frac{1}{2p+2},\frac{1}{2}-d\right)+d\right\}\frac{1}{\gamma+1}},\]
\begin{align*}
&G_2(\lambda)\\
& =bn^{\left\{-\min\left(\frac{1}{2p+2},\frac{1}{2}-d\right)+d\right\}\frac{\gamma}{\gamma+1}}\\
&\quad+\left\{\frac{3r(c_{2,n})^{-1}}{c}
 n^{\left\{-\min\left(\frac{1}{2p+2},\frac{1}{2}-d\right)+d\right\}\frac{-1}{\gamma+1}}+\frac{1}{c}\right\}
 \left\{3\left(\frac{2a}{n}\right)^{\frac{1}{2p+2}}+\left(\frac{2\eta}{n}\right)^{\frac{1}{2}}
 +\frac{b_{3}(\eta)n^{-\frac{1}{2}}}{c_{n,L}}
 +\frac{4c}{3}b_3(\eta)n^{-\frac{1}{2}}\right\}\\
 &\leq bn^{\left\{-\min\left(\frac{1}{2p+2},\frac{1}{2}-d\right)+d\right\}\frac{\gamma}{\gamma+1}}\\
 &\quad+c_{\text{P}}^{\ast}\left(c_{a}^{\ast}+\sqrt{\eta}+2b_3(\eta)\right)
 n^{\left\{-\min\left(\frac{1}{2p+2},\frac{1}{2}-d\right)+d\right\}\frac{-1}{\gamma+1}
 -\min\left(\frac{1}{2p+2},\frac{1}{2}-d\right)+d}\\
 &\leq Q^{\ast}\left\{c_{a,b}^{\ast}+\sqrt{\eta}+2b_3(\eta)\right\}
n^{\left\{-\min\left(\frac{1}{2p+2},\frac{1}{2}-d\right)+d\right\}\frac{\gamma}{\gamma+1}}
\end{align*}
where $c_{a}^{\ast}$ is a constant related to $a$,
$c_{a,b}^{\ast}$ is a constant related to $a$ and $b$,
$c_{\text{\prob}}^{\ast}$ and $Q^{\ast}$ are constants related to $\prob$.
None of them is related to $\eta$.

Therefore, for fixed $\lambda>0$, $n\geq1$, $\varepsilon>0$, and $\eta>0$, with probability not less
than $1-3e^{\eta}$,
\[A^{DR}(\lambda)\leq Q^{\ast}\left\{c_{a,b}^{\ast}+\sqrt{\eta}+b_1(\eta)+b_3(\eta)\right\}
n^{\left\{-\min\left(\frac{1}{2p+2},\frac{1}{2}-d\right)+d\right\}\frac{\gamma}{\gamma+1}}.\]
The obtained learning rate is $n^{\left\{-\min\left(\frac{1}{2p+2},\frac{1}{2}-d\right)+d\right\}\frac{\gamma}{\gamma+1}}                                                                                                                    $.

Case 2: $\left|\widehat{\beta}-\beta_{0}\right|=\O_p\left(n^{-\frac{1}{2}}\right)$.
By \eqref{learnRateDR2} and Assumption~\ref{assp:BoundA2},
\begin{align*}
 A^{DR}(\lambda)
 &\leq b\lambda^{\gamma}+\frac{3\left\{r(c_{2,n}\lambda)^{-1}+1\right\}}{c}
 \left\{3\left(\frac{2a}{n}\right)^{\frac{1}{2p+2}}+\left(\frac{2\eta}{n}\right)^{\frac{1}{2}}\right\}\\
 &\quad+\frac{2(1+c_{n,U})}{c_{n,L}}b_{2}(\eta)n^{-\left(\frac{1}{2}-d\right)}\lambda^{-1}
 +\frac{1}{c\cdot c_{n,L}}b_{4}(\eta)n^{-\left(\frac{1}{2}-\frac{d}{2}\right)}\lambda^{-\frac{1}{2}}.
\end{align*}
Choosing $0<\lambda<1$, we have
\begin{align*}
 A^{DR}(\lambda)
 &\leq b\lambda^{\gamma}+\frac{3\left\{r(c_{2,n}\lambda)^{-1}+1\right\}}{c}\left\{
 3\left(\frac{2a}{n}\right)^{\frac{1}{2p+2}}+\left(\frac{2\eta}{n}\right)^{\frac{1}{2}}\right\}\\
 &\quad+\frac{2(1+c_{n,U})}{c_{n,L}}b_{2}(\eta)n^{-\left(\frac{1}{2}-d\right)}\lambda^{-1}
 +\frac{1}{c\cdot c_{n,L}}b_{4}(\eta)n^{-\left(\frac{1}{2}-\frac{d}{2}\right)}\lambda^{-1}.
\end{align*}

Let
\begin{align*}
G_3(\lambda)&= b\lambda^{\gamma}+\frac{3\left\{r(c_{2,n}\lambda)^{-1}+1\right\}}{c}
\left\{3\left(\frac{2a}{n}\right)^{\frac{1}{2p+2}}+\left(\frac{2\eta}{n}\right)^{\frac{1}{2}}\right\}\\ &\quad+\frac{2(1+c_{n,U})}{c_{n,L}}b_{2}(\eta)n^{-\left(\frac{1}{2}-d\right)}\lambda^{-1}
 +\frac{1}{c\cdot c_{n,L}}b_{4}(\eta)n^{-\left(\frac{1}{2}-\frac{d}{2}\right)}\lambda^{-1}.
\end{align*}
Taking the derivative with respect to $\lambda$ and setting it equal to $0$,
\begin{align*}
b\gamma\lambda^{\gamma-1}&=\frac{3r(c_{2,n})^{-1}}{c}\lambda^{-2}
\left\{3\left(\frac{2a}{n}\right)^{\frac{1}{2p+2}}+\left(\frac{2\eta}{n}\right)^{\frac{1}{2}}\right\}
+\frac{2(1+c_{n,U})}{c_{n,L}}b_{2}(\eta)n^{-\left(\frac{1}{2}-d\right)}\lambda^{-2}\\
&\quad +\frac{1}{c\cdot c_{n,L}}b_{4}(\eta)n^{-\left(\frac{1}{2}-\frac{d}{2}\right)}\lambda^{-2}.
\end{align*}
Then,
\[\lambda^{\gamma+1}
\propto\left(\frac{1}{n}\right)^{\min\left(\frac{1}{2p+1},\frac{1}{2}-d\right)-d} \\
\Rightarrow\lambda\propto n^{\left\{-\min\left(\frac{1}{2p+1},\frac{1}{2}-d\right)+d\right\}\frac{1}{\gamma+1}}.\]
Note that by choosing large $r$, we have $G_3''\left(n^{\left\{-\min\left(\frac{1}{2p+1},\frac{1}{2}-d\right)+d\right\}\frac{1}{\gamma+1}}\right)>0$.

Then
\begin{align*}
&G_3(\lambda)\\
& =bn^{\left\{-\min\left(\frac{1}{2p+1},\frac{1}{2}-d\right)+d\right\}\frac{\gamma}{\gamma+1}}
  +\left\{\frac{3r(c_{2,n})^{-1}}{c}
n^{\left\{-\min\left(\frac{1}{2p+1},\frac{1}{2}-d\right)+d\right\}\frac{-1}{\gamma+1}}+\frac{1}{c}\right\}
\left\{3\left(\frac{2a}{n}\right)^{\frac{1}{2p+2}}+\left(\frac{2\eta}{n}\right)^{\frac{1}{2}}\right\}\\
&\quad+\frac{1+c_{n,U}}{c_{n,L}}b_{2}(\eta)n^{-\left(\frac{1}{2}-d\right)}
n^{\left\{-\min\left(\frac{1}{2p+1},\frac{1}{2}-d\right)+d\right\}\frac{-1}{\gamma+1}}
 +\frac{2}{c\cdot c_{n,L}}b_{4}(\eta)n^{-\left(\frac{1}{2}-\frac{d}{2}\right)}
 n^{\left\{-\min\left(\frac{1}{2p+1},\frac{1}{2}-d\right)+d\right\}\frac{-1}{\gamma+1}}\\
 &\leq bn^{\left\{-\min\left(\frac{1}{2p+1},\frac{1}{2}-d\right)+d\right\}\frac{\gamma}{\gamma+1}}\\
 &\quad+c_{\text{P}}^{\star}\left(c_{a}^{\star}+\sqrt{\eta}+b_2(\eta)+b_4(\eta)\right)
n^{\left\{-\min\left(\frac{1}{2p+1},\frac{1}{2}-d\right)+d\right\}\frac{-1}{\gamma+1}
-\min\left(\frac{1}{2p+1},\frac{1}{2}-d\right)+d}\\
 &\leq Q^{\star}\left(c_{a,b}^{\star}+\sqrt{\eta}+b_2(\eta)+b_{4}(\eta)\right)
 n^{\left\{-\min\left(\frac{1}{2p+1},\frac{1}{2}-d\right)+d\right\}\frac{\gamma}{\gamma+1}},
\end{align*}
where $c_{a}^{\star}$ is a constant related to $a$,
$c_{a,d}^{\star}$ is a constant related to $a$ and $b$,
$c_{\text{\prob}}^{\star}$ and $Q^{\star}$ are constants related to $\prob$.
None of them is related to $\eta$.

Therefore, for fixed $0<\lambda<1$, $n\geq1$, $\varepsilon>0$, and $\eta>0$, with probability not less
than $1-3e^{\eta}$,
\[A^{DR}(\lambda)\leq Q^{\star}\left(c_{a,b}^{\star}+\sqrt{\eta}+b_2(\eta)+b_{4}(\eta)\right)
n^{\left\{-\min\left(\frac{1}{2p+1},\frac{1}{2}-d\right)+d\right\}\frac{\gamma}{\gamma+1}}.\]
The obtained learning rate is
$n^{\left\{-\min\left(\frac{1}{2p+1},\frac{1}{2}-d\right)+d\right\}\frac{\gamma}{\gamma+1}}$.
\end{proof}
\bibliographystyle{plainnat}

\end{document}